%% file: weisz22.tex
\RequirePackage{etoolbox}

\PassOptionsToPackage{hypertexnames=false}{hyperref}

\documentclass[12pt]{alt2022} 
\hypersetup{
    colorlinks=true,
    linkcolor=blue,
    filecolor=magenta,      
    urlcolor=cyan,
    citecolor=blue,
    pdftitle={TensorPlan and the Few Actions Lower Bound for Planning in MDPs under Linear Realizability of Optimal Value Functions},
    pdfpagemode=FullScreen,
    }

\usepackage{amsthm}

\usepackage{amsmath}
\usepackage{cases}

\newcommand*{\ditto}{---\texttt{"}---}
\newcommand{\cmark}{\ding{51}}%
\newcommand{\xmark}{\ding{55}}%
\newcommand{\thetafinal}{\theta^+}
 
 \newcommand{\B}{\mathcal{B}}

 \newcommand{\sol}{\mathrm{Sol}}
 \newcommand{\concat}[1]{\overline{#1}}

 \newcommand{\tensorplan}{{\sc TensorPlan}\xspace}

 \newcommand{\mtpdetq}{\mathrm{TP'}}
 
 \newcommand{\approxmeasure}{{\sc ApproxTD}\xspace}

 \newcommand{\ordo}{\mathcal{O}}
 \newcommand{\aplan}{\mathcal{A}}

 \newcommand{\Bernoulli}{\mathrm{Ber}}
 
 \newcommand{\ltheta}{{\theta}}
 \newcommand{\Wc}{{W_{\mathrm{close}}}}
 \newcommand{\BigTheta}{{\Theta}}
 \newcommand{\minsteps}{\ceil{\sd/4}}
 \newcommand{\fix}{\mathrm{fix}}
 \newcommand{\notfix}{\neg{\mathrm{fix}}}
 \newcommand{\efix}{\mathrm{e}^{\fix}}
 \newcommand{\enotfix}{\mathrm{e}^{\notfix}}
 \newcommand{\sreach}{\cS_{\mathrm{r}}}%
\newcommand{\snotreach}{\cS_{\neg{\mathrm{r}}}}%
 \newcommand{\ctflip}{\mathrm{ct}^{\mathrm{flip}}}
 \DeclareMathOperator{\diff}{diff}
 \newcommand{\phiq}{\phi_q}
 \newcommand{\phiv}{\phi_v}
 \newcommand{\barltheta}{\bar{\ltheta}}
\newcommand{\barphiv}{\bar{\phi}_v}

 \DeclareMathOperator{\reach}{Reach}
 
\newcommand{\bfitDelta}{X}  
\newcommand{\vmin}{\land}
\newcommand{\vmax}{\lor}
\newcommand{\ld}{d}
\newcommand{\sd}{{p}}
\newcommand{\sgamma}{{\sd/2}}
\newcommand{\stheta}{{w}}  
\newcommand{\ntheta}{{{\bar{\stheta}}}}  
\newcommand{\CleanTest}{{\tt CleanTest}}
\newcommand{\true}{{\tt true}}
\newcommand{\false}{{\tt false}}
\newcommand{\simm}{\simulatesc\xspace}
\newcommand{\EpisodeStart}{{\tt EpisodeStart}}
 \newcommand{\ldd}{\overline{\Delta}}
 \newcommand{\hd}{\hat{\Delta}}
 
 \usepackage{algpseudocode}
\usepackage{algorithmicx}
\usepackage{varwidth}

\algnewcommand{\algorithmicgoto}{\textbf{goto}}%
\algnewcommand{\Goto}[1]{\algorithmicgoto~\ref{#1}}%
\algnewcommand{\Break}{\textbf{break}}%

\algnewcommand{\Initialize}[1]{%
  \State \textbf{Initialize:}
  \Statex \hspace*{\algorithmicindent}\parbox[t]{.8\linewidth}{\raggedright #1}
}
\algnewcommand{\Inputs}[1]{%
  \State \textbf{Inputs:}
  \Statex \hspace*{\algorithmicindent}\parbox[t]{.8\linewidth}{\raggedright #1}
}

\definecolor{LightGreen}{rgb}{0.8,1,0.8}
\definecolor{LightRed}{rgb}{1,0.8,0.8}
\definecolor{purplish}{rgb}{0.05,0.4,0.7}

\input{preamble}

\makeatletter
 \let\Ginclude@graphics\@org@Ginclude@graphics 
\makeatother

\usepackage{times}
\usepackage{newtxmath}

\usepackage{thmtools}
\usepackage{thm-restate}

\usepackage[utf8]{inputenc}

\begin{document}

\title[TensorPlan and the Few Actions Lower Bounds]{TensorPlan and the Few Actions Lower Bound for Planning in MDPs under Linear Realizability of Optimal Value Functions}~

\altauthor{%
 \Name{Gell\'ert Weisz}\\
 \addr{DeepMind, London, UK}\\
 \addr{University College London, London, UK}
 \AND
 \Name{{Cs}aba Szepesv\'ari}\\
 \addr{DeepMind, London, UK}\\
 \addr{University of Alberta, Edmonton, Canada}
 \AND
 \Name{Andr\'as Gy\"orgy}\\
 \addr{DeepMind, London, UK}
}
\maketitle

\begin{abstract}
We consider the minimax query complexity of online planning with a generative model in fixed-horizon Markov decision processes (MDPs) with linear function approximation. 
Following recent works,
we consider broad classes of problems where  
either 
(i) the optimal value function $v^\star$ or 
(ii) the optimal action-value function $q^\star$ lie in the linear span of some features; or
 (iii) both $v^\star$ and $q^\star$ lie in the linear span when restricted to the states reachable from the starting state. 
Recently, \citet{weisz2020exponential} showed that under (ii) the minimax query complexity of any planning algorithm 
is at least exponential in the horizon $H$ or in the feature dimension $d$ when the size $A$ of the action set can be chosen to be exponential in $\min(d,H)$. 
On the other hand, for the setting (i), \citet{weisz2021query} introduced TensorPlan, a planner whose query cost is polynomial in all relevant quantities when the number of actions is fixed.
Among other things, these two works left open the question whether polynomial query complexity is possible when $A$ is subexponential in $\min(d,H)$. 
In this paper we answer this question in the negative: we show that an exponentially large lower bound holds when $A=\Omega( \min(d^{1/4},H^{1/2}))$, under either (i), (ii) or (iii). In particular, this implies a perhaps surprising exponential separation of query complexity compared to the work of \citet{du2021bilinear} who prove a polynomial upper bound when (iii) holds for all states. Furthermore, we show that the upper bound of TensorPlan can be extended to hold under (iii) and, for MDPs with deterministic transitions and stochastic rewards, also under (ii).
\end{abstract}


\section{Introduction}
We are concerned with the query complexity of \textbf{online planning} in fixed-horizon Markov decision processes (MDPs) 
 with large state spaces and finite action sets when the planner is used in a \textbf{closed-loop} configuration.
In each step of the closed-loop process, the planner is \textbf{called} with the state of the process,
after which it is allowed to query a simulator of the MDP until it decides to stop. Then it needs to return an action, which is used to move the state of the process. 
This is repeated until the time for the episode runs out.
The goal of the planner is to maximize the total expected reward incurred in the episode. 

To deal with large state spaces, 
the planner is helped by giving it access to features (elements of $\mathbb{R}^d$, the $d$-dimensional Euclidean space). The features are associated with the states, but the planner is only given the features of that states that it encounters either in response to a query or when the planner is called with a state.
The features are assumed to be such that the optimal value $v^\star(s)$ at any state $s$ is equal to 
the linear combination of the features at the state, where the vector $\theta^\star\in \mathbb{R}^d$ formed by the coefficients is fixed (regardless of the state) but unknown. An MDP together with a feature-map (regardless of whether $v^\star$ is realizable or not) is called a \textbf{featurized MDP}. 
We also consider the analogous case when the linear combination of features, which can now also depend on the actions, gives the optimal action-value function $q^\star$, 
as well as the case when the planner has access to two sets of features, one for representing the optimal value function and another one for representing the optimal action-value function. 

The model described so far, regardless of which value functions are realizable, is called planning with \textbf{local access} as the planner can query the simulator for transitions and associated features only at states previously encountered. 
We also consider planning with \textbf{global access} to the features where the planner is given 
the set of all the states and associated features in advance (no queries required), and the option to query the simulator for transitions at any state of its choice. %

Fix a set of featurized MDPs $\cM$ and a positive real $\delta$.
We say that a planner $P$ is \textbf{$\delta$-sound} for $\cM$ if 
for any $M\in \cM$ and any start state $s$ of the underlying MDP,
the total expected value the planner achieves in an episode 
when used in $M$ from $s$ 
is at most $\delta$ worse than the optimal value associated with that start state. 
The set of $\delta$-sound planners for $\cM$ under the global access (local access) model is denoted by 
$\cP_{\mathrm{GA}}(\cM, \delta)$ (respectively, by $\cP_{\mathrm{LA}}(\cM, \delta)$).

A planner's \textbf{query cost} is the worst-case expected number of queries it ever uses in a call.
For the global access model and a featurized MDP $M$ this is denoted by $q_{\mathrm{GA}}(P,M)$, while for the local access model this is denoted by $q_{\mathrm{LA}}(P,M)$.

Our main results are concerned with the \textbf{minimax query cost} that $\delta$-sound planners incur. 
In particular, for a class of featurized MDPs $\cM$ and $\delta>0$, we denote by $\cC^\star_{\mathrm{GA}}(\cM,\delta)$ the minimax query cost (in short, query complexity) of $\delta$-sound planners over $\cM$ given the global access model: 
\[
\cC^\star_{\mathrm{GA}}(\cM,\delta) = \inf_{P\in \cP_{\mathrm{GA}}(\cM, \delta)} \sup_{M\in \cM} q_{\mathrm{GA}}(P,M).
\]
Similarly, we define
\[
\cC^\star_{\mathrm{LA}}(\cM,\delta) = \inf_{P\in \cP_{\mathrm{LA}}(\cM, \delta)} \sup_{M\in \cM} q_{\mathrm{LA}}(P,M).
\]
Local access is more demanding in that any planner that is $\delta$-sound for $\cM$ under the local access model is automatically $\delta$-sound for $\cM$ under the global access model. It follows that 
\begin{align}
\cC^\star_{\mathrm{GA}}(\cM,\delta)\le \cC^\star_{\mathrm{LA}}(\cM,\delta)
\label{eq:qclaga}
\end{align}
no matter the choice of $\cM$ and $\delta$.

In this paper we are concerned mainly with three classes of featurized MDPs. For $\thetabound\ge 0$ and positive integers $d,H,A$, these are defined as follows:
\begin{itemize}
\item $v^\star$-realizable class: $\cM^{v^\star}_{\thetabound,d,H,A}$ is the class of finite-state-space featurized MDPs with $A$ actions, 
where the feature-vectors are $d$-dimensional, the length of the episodes is $H$.
For any $(M,\phi)$ in this class, 
$M$ is an MDP with some state space $\cS$ and 
random rewards confined to (say)  $[0,1]$, 
the associated feature-map $\phi:\cS \to \R^d$ with $\sup_{s\in\cS}\norm{\phi(s)}_2\le 1$ is such that  for some  $\theta^\star\in \R^d$ with $\| \theta^\star\|_2\le \thetabound$, 
$v^\star_M(s)=\phi(s)^\top \theta^\star$ holds for all $s\in \cS$ where $v^\star_M$ is the optimal value function in $M$.
\footnote{For the sake of simplifying the notation, 
we assume that the states encode the stage index that the process can be at within an episode and we also add a final absorbing state where all actions incur a zero reward. This allows us to use the total expected reward criterion and thus a notation where the dependence on the stage index of values can be suppressed, and also means that we can talk about the initial states in an MDP.}
\item $q^\star$-realizable class: $\cM^{q^\star}_{\thetabound,d,H,A}$ is the class of featurized MDPs as above except that here 
for any $(M,\phi)$ in the class, for $[A]:=\{1,\dots,A\}$, $\phi: \cS \times [A] \to \R^d$ with $\sup_{s\in\cS,a\in[A]}\norm{\phi(s,a)}_2\le 1$ and $\theta^\star\in \R^d$ with $\| \theta^\star\|_2\le \thetabound$, we now require that $q^\star_M(s,a) = \phi(s,a)^\top \theta^\star$ holds for all states $s\in \cS$ and actions $a\in [A]$, where $q^\star_M(s,a)$ is the optimal action-value at $(s,a)$.
\item Reachable-$v^\star/q^\star$-realizable class: $\cM^{v^\star/q^\star \mathrm{ reach}}_{\thetabound,d,H,A}$ is the class of featurized MDPs
as above except that here the MDPs $M$ are associated with two feature-maps, $\phiv: \cS \to \R^d$ and $\phiq:\cS \times [A] \to \R^d$ (their 2-norms bounded by 1 as before), and it is assumed that
there exists some $\theta^\star\in \R^d$ with $\|\theta^\star\|_2\le \thetabound$
such that
 $v^\star_M(s)= \phiv(s)^\top \theta^\star$ and $q^\star_M(s,a) = \phiq(s,a)^\top \theta^\star$ hold for any action $a$ and any state $s$ of the MDP that is reachable from the initial states.
 \footnote{It is without loss of generality that we use that same $\theta^\star$ in the inner products that yield $v^\star_M(s)$ and $q^\star_M(s,a)$: if these parameters are not shared, we can concatenate them with only a factor 2 increase in $d$ and $B$.}
\end{itemize}
Our main results are a lower and an upper bound for the query complexity of planning for $\cM$ where $\cM$ is one of the above classes. The lower bound is for the case when the number of actions grows polynomially with $d\wedge H$, where we define $a\vmin b := \min(a,b)$:
\begin{theorem}[Lower bound with global access, at least $\mathrm{poly}(d\wedge H)$ actions]
\label{thm:lb}
For $\delta,\thetabound,d,H$ sufficiently large, $A\ge d^{1/4}\wedge H^{1/2}$,
\[
\cC^\star_{\mathrm{GA}}(\cM\cap \cM^{\mathrm{Pdet}},\delta) = 2^{\Omega(d^{1/4}\wedge H^{1/2})}\,,
\]
where 
$\cM^{\mathrm{Pdet}}$ is the class of featurized MDPs with deterministic transitions and
\[
\cM \in \{ \cM^{v^\star}_{\thetabound,d,H,A}, \cM^{q^\star}_{\thetabound,d,H,A}, \cM^{v^\star/q^\star \mathrm{ reach}}_{\thetabound,d,H,A} \}\,.
\]
\end{theorem}
Together with \eqref{eq:qclaga} this result also gives a lower bound on the query complexity of planning when only local access is available to the featurized MDP.

We also prove a result that provides a polynomial upper bound on the corresponding query complexity for a fixed number of actions.
\begin{theorem}[Upper bound with local access]
\label{thm:ub}
For 
$\cM \in \{ \cM^{v^\star}_{\thetabound,d,H,A},
 \cM^{q^\star}_{\thetabound,d,H,A} \cap \cM^{\mathrm{Pdet}},
  \cM^{v^\star/q^\star \mathrm{ reach}}_{\thetabound,d,H,A} \}$, arbitrary positive reals $\delta,\thetabound$ and arbitrary positive integers $d,H$,
\[
\cC^\star_{\mathrm{LA}}(\cM,\delta) = O\Big(\mathrm{poly}\Big( \big(\tfrac{dH}{\delta}\big)^A, \thetabound \Big)\Big)\,.
\]
\end{theorem}

The rest of the paper is organized as follows: 
In the next section, we discuss these result and their relationship to existing work. This is followed by 
Section~\ref{sec:prelim} that introduces our notation and gives the precise problem description. 
The proof of Theorem~\ref{thm:lb} is presented in Section~\ref{sec:lb}:
after a brief intuitive overview in Section~\ref{sec:proof-overview}, 
for modularity, a simplified ``abstract game'' is introduced and is shown to be hard to solve in Section~\ref{sec:abstract-game}; the remaining proof turns the abstract game into a featurized MDP with a low action count that is similarly hard to solve for a planner;
this is summarized in Section~\ref{sec:mdp-summary} before the arguments are presented in detail.
Finally, the proof of Theorem~\ref{thm:ub} is given in Section~\ref{sec:ub}.

\section{Discussion and related works}
\label{sec:disc}
A great many problems of  interest can be formulated as optimal sequential decision making in a stochastic environment. 
If the model of the environment is given or learned with a sufficient accuracy, one only has to figure out how to use the model to find good actions. This is the problem addressed in planning.
An elegant, minimalist approach to describe stochastic controlled environments is to adopt the language of
MDPs.
The price of simplicity (and thus generality) is that efficient planning in large state-spaces is intractable, a phenomenon pointed out by \citep{bellman57} 
and today informally referred to as Bellman's curse of dimensionality. 
While dynamic programming methods in MDPs with $S$ states, $A$ actions and a horizon of $H$ can solve the planning problem with $\mathrm{poly}(S,A,H)$ resources 
\citep{Tseng90,Ye11:MOR,Scherrer2016-gd},
in the lack of extra information, 
the query cost of the easier problem of online planning
\citep[Chapter 6][]{kolobov2012planning}
 is at least $\Omega(A^H)$ when the number of states is unbounded \citep{kearns2002sparse}.
An intriguing approach
to avoid intractability when both $S$ and $H$ are large is the use of ``function approximation'' which promises to empower planners to extrapolate beyond the states that the planner has encountered.
This approach has been proposed
shortly after MDPs have been introduced when it was observed that in various problems of practical interest,
value functions that the dynamic programming algorithms aim to compute can be well approximated with the linear combination of only a few basis functions, which themselves can be guessed by studying the structure of the problem to be solved \citep{BeKaKo63,SchSei85}.
This raises the question of whether under such a favorable condition a provably efficient planner exist, i.e., whether the curse can be sidestepped.

While this question was arguably one of the main driving forces behind much of the research in operations research and reinforcement learning since the beginnings, most
of the early results focused on the case when the function space underlying the features have a certain completeness property when dynamic programming algorithms can be successfully adopted \citep[e.g.,][]{BeTs96,tsitsiklis1996feature,Munos03,Munos05,szemu:avi2005}. %
For more recent works in this, and some other related directions, see, e.g., 
\citep{Du_Kakade_Wang_Yan_2019,LaSzeGe19,du2021bilinear} and the references therein.

While interesting, these works left open the question of whether efficient planners exist in the case when the function space may lack the completeness property but is still able to represent the optimal value function.
The first results in this direction are quite recent
\citep{Wen_Roy_2013,weisz2020exponential,weisz2021query,du2021bilinear}.
Of these, the closest to our results is a result of 
\citet{weisz2020exponential} who proved an exponential lower bound:
\begin{theorem}[\citealp{weisz2020exponential}, Theorem 9, lower bound for exponentially many actions]\label{thm:exp-lb-exp-a}
For any $\delta>0$ sufficiently small,
and $H,d,\thetabound$ sufficiently large, if
$A=2^{\Omega(d\wedge H)}$ %
then \footnote{Recall that $a\wedge b=\min(a,b)$.}
\[
\cC^\star_{\mathrm{GA}}(\cM^{q^\star}_{\thetabound,d,H,A}\cap \cM^{\mathrm{Pdet}},\delta)  =  2^{\Omega(\ld\vmin H)}\,.
\]
\end{theorem}
According to this result,
 as long as there are exponentially many actions,
planning remains intractable even
for featurized MDPs where the features provided realize the optimal action-value function of the associated MDP
and even if the MDPs are deterministic.
Note that the exponential lower bound in this theorem is nontrivial since the query complexity of finding a good approximation to a function that lies in the span of $d$ features from input-output examples is polynomial in the number of features regardless the cardinality of the input domain of the function, %
hence, the intractability in the above result cannot be solely attributed to the presence of a large action set.

\begin{quotation}
\emph{
Thus, an intriguing question is whether planning for the same setting as considered by that of Theorem~\ref{thm:exp-lb-exp-a}
but with subexponential number of actions is tractable.}
\end{quotation}

This question is partially answered by
Theorem~\ref{thm:lb}, which states that even with an action count that is polynomial in $d$ and $H$, planning remains intractable for the same class of MDPs.
This theorem also extends the result to two additional settings.
The remaining problem then is 
whether planning is tractable when $2\le A = o(d^{1/4}\wedge H^{1/2})$.
For a constant number of actions, Theorem~\ref{thm:ub} answers this question in the positive, though part of this theorem
when $\cM = \cM^{v^\star}_{\thetabound,d,H,A}$ has been proved earlier by \citet{weisz2021query} and in fact the proof of 
Theorem~\ref{thm:ub} for the remaining two cases follows closely their proof.
The importance of
these extensions we give here is that they complement the lower bounds we prove. 
(Sadly, the featurized MDP classes used in these results are incomparable in the sense that 
we know of no general way of transforming results from one class to another, hence the need for the separate proofs.)

\begin{quotation}
\emph{
Intriguingly, even the new results leave open whether online planning with local access is tractable under $q^\star$ realizability when the MDPs involved have \textbf{stochastic transition dynamics and rewards} while the number of actions is fixed.}
\end{quotation}

In fact, even though our upper bound holds generally for the $v^\star$-realizable and reachable-$v^\star/q^\star$-realizable classes of MDPs, for the $q^\star$-realizable class our upper bound only holds for MDPs with deterministic transitions. %
If in addition to the transitions, the rewards are also deterministic,
the result of \citet{Wen_Roy_2013} can be used to show a polynomial query (and even computational) complexity for online planning with local access.
While the compute cost would depend linearly on the number of actions, the number of actions would not even appear in the query cost. 
(This result can also be proved directly by arguing that 
the subspace that contains $q^\star$ is either trivial, or a rollout with any action sequence will provide new information that can be used to decrease the dimension of this subspace by one.)

Given that our upper bound is polynomial when the number of actions is fixed, one may speculate that when the number of actions is large, perhaps one should replace each stage of an episode with $\log_2(A)$ stages, where actions would be chosen by determining their bits one by one, in a sequential fashion. The difficulty then is that this calls for an extension of the state space and a new, suitable feature-map. If one could derive a new, suitable feature-map given the old one and other information available during planning,  this would result in a planner with query complexity $O(\mathrm{poly}( \tilde{d}H\log_2(A)/\delta, \thetabound ))$, showing a very mild dependence on the number of actions provided that $\tilde{d}$, the dimensionality of the new feature-map can be kept small. Sadly, our lower bound tells us that this \textbf{action binarization} approach cannot work when the number of actions is at least $\Omega(d^{1/4}\wedge H^{1/2})$.

\begin{table}[t]
\centering
\begin{tabular}{|c|c|c|c|c|}
\hline
Publications & Action count & MDP class & $\poly(\cdot)$ sample  \\
      &     &            & complexity? \\ \hline
\rowcolor{LightGreen} \citet{Wen_Roy_2013} & any & $\cM^{q^\star}_{\thetabound,d,H,A}\cap \cM^{\mathrm{det}}$ & \cmark  \\ \hline 
\rowcolor{LightGreen} \citet{du2021bilinear} & any & $\cM^{v^\star/q^\star}_{\thetabound,d,H,A}$ & \cmark  \\ \hline 
\rowcolor{LightGreen} \citet{weisz2021query} & $\ordo(1)$ & $\cM^{v^\star}_{\thetabound,d,H,A}$ & \cmark  \\ \hline 
\rowcolor{LightRed} \citet{weisz2020exponential} & $2^{\Omega(\ld\wedge H)}$ & $\cM^{q^\star}_{\thetabound,d,H,A}\cap \cM^{\mathrm{Pdet}}$ & \xmark \\ \hline \hline
\rowcolor{LightRed} This work & $\Omega(\ld^{1/4}\wedge H^{1/2})$ & $\cM^{q^\star}_{\thetabound,d,H,A}\cap \cM^{\mathrm{Pdet}}$ & \xmark  \\ \hline
\rowcolor{LightRed}     \ditto     &         \ditto                    & $\cM^{v^\star}_{\thetabound,d,H,A}\cap \cM^{\mathrm{Pdet}}$ & \xmark  \\ \hline
\rowcolor{LightRed}     \ditto     &         \ditto                    &$\cM^{v^\star/q^\star \mathrm{ reach}}_{\thetabound,d,H,A}\cap \cM^{\mathrm{Pdet}}$  & \xmark  \\ \hline
\rowcolor{LightGreen} \ditto     & $\ordo(1)$ & $\cM^{q^\star}_{\thetabound,d,H,A}\cap \cM^{\mathrm{Pdet}}$ & \cmark \\ \hline 
\rowcolor{LightGreen} \ditto     & \ditto & $\cM^{v^\star}_{\thetabound,d,H,A}$ & \cmark \\ \hline 
\rowcolor{LightGreen} \ditto     & \ditto & $\cM^{v^\star/q^\star \mathrm{ reach}}_{\thetabound,d,H,A}$ & \cmark \\ \hline 
\end{tabular}
\caption{Comparison of various query complexity results for online planning with global access, and features realizing the optimal value or action-value function. 
The symbol $ \cM^{\mathrm{det}}$ stands for the class of finite MDPs with deterministic transitions and rewards.
\cmark~indicates the existence of a sound planner with query cost polynomial in relevant parameters (excluding $S$ and $A$); \xmark~indicates that such a planner does not exist.}
\label{tab:comparison-of-prior-work}
\end{table}

Recently, the topic of online learning with good features has also seen many new results. As opposed to planning, here there is no simulator that can reset the state, and the only reset possible is to the start state of the MDP.
As such, this is a harder setting than online planning.
We would like to emphasize two results in this topic closely related to Theorem~\ref{thm:ub}.
First, the work of \citet{du2021bilinear} implies that %
as long as all the features are given in advance (as in global access), regardless the number of actions $A$, the class 
$\cM^{v^\star/q^\star}_{\thetabound,d,H,A}$ enjoys a minimax query complexity of order $\mathrm{poly}(\thetabound,d,H)$.
For this result, we define the class 
$\cM^{v^\star/q^\star}_{\thetabound,d,H,A}$
like the class
$\cM^{v^\star/q^\star \mathrm{ reach}}_{\thetabound,d,H,A}$,
except that realizability is required to hold over the \textbf{entire state-space}, and not only for states reachable from the initial states.
Note that while according to Theorem~\ref{thm:lb}, 
online planning with \textbf{global (and thus also local) access} over $\cM^{v^\star/q^\star \mathrm{ reach}}_{\thetabound,d,H,A}$ is \textbf{intractable},
the result just mentioned implies that online planning with \textbf{global access} over $\cM^{v^\star/q^\star}_{\thetabound,d,H,A}$ is \textbf{tractable}.
Thus, while it is immediate from the definitions that 
\[
\cM^{v^\star/q^\star}_{\thetabound,d,H,A} \subset \cM^{v^\star/q^\star \mathrm{ reach}}_{\thetabound,d,H,A},
\]
the two results together imply 
that the class on the right-hand side (RHS) is  substantially larger than the one on the left-hand side (LHS).
However, this difference disappears if in addition to the feature-maps, the planner 
working with the class on the RHS 
is also provided with the reachable subset of the state-space.
Indeed, in this case, the MDP that the planner works with can be redefined by removing all the unreachable states, leaving a featurized MDP that belongs to the class on the LHS.
The two results together thus indicate that this extra piece of knowledge is hard to obtain.
Indeed, as this information is not provided in the \textbf{local access} model when the state space is redefined to reachable states only, Theorem~\ref{thm:lb}
implies that online learning with \textbf{local access} is \textbf{intractable}, which proves the first exponential information theoretic separation result between local and global access that the authors are aware of (cf. \citealp{du2021bilinear}).

The fact that the set of reachable states is hard to obtain should not be too surprising. In robotics and many other problems the reachable part of the state space is known to have a fairly complicated geometry.  In such applications, assuming that the features describe the value functions over the whole state space means that they encode the complicated shape of these reachable sets, as well. As this looks difficult to achieve, the class that requires realizability only for reachable states appears more attractive
even though for this class, per our lower bound, tractability only holds when the number of actions is relatively small.
For convenience, we summarize the results discussed so far in Table~\ref{tab:comparison-of-prior-work}.

The second result of interest in online learning with good features is due to \citet{wang2021exponential}, who gave an exponential lower bound in the flavor of Theorem~\ref{thm:exp-lb-exp-a}
by adapting the hard MDP construction of \citet{weisz2020exponential} to satisfy a \textbf{constant suboptimality gap} between the action values of the best and second-best actions for all states.
Instead of exponentially downscaling the values in the more advanced stages, such a result is possible by implementing this reduction effect through zero-reward transitions to the episode-over stage, such that the probability of reaching an advanced stage (instead of the value at such a stage) is exponentially small.
While this does not lead to a hardness result in our online planning setup \citep[at least under global access, see e.g.,][]{Du_Kakade_Wang_Yan_2019}, but we note that we expect similar modifications to the hard MDP class underlying our Theorem~\ref{thm:lb} to lead to a similar, constant suboptimality gap version of the theorem in the online learning case. %

At this stage, the reader may also wonder about the limitations of the present work.
For example, we restricted the MDPs to those that have finite state spaces.
As it turns out, this was done only for the sake of convenience: on the one hand, the lower bounds are not impacted by this restriction,
while, on the other hand, the upper bounds go through with no change to the proofs apart from the occasional need to switch to a more technical, measure-theoretic language.
Similar comments apply to our MDP model which assumes that the same number of actions is available at all states,
or changing the problem to allow misspecification errors (when the features can only represent the target functions only with some positive error).
If one allows for such misspecification errors, the misspecification level $\eta$ (measured by the maximum norm) 
will put a lower bound on the suboptimality gap $\delta$ that the planner can achieve.
The argument of \citet{weisz2021query} can be used in this case and gives that with a polynomial query complexity, one can achieve
suboptimality gaps of magnitude $\delta = \Omega(\mathrm{poly}(H,d) \eta^{1/A})$.
While (similarly) the upper bound of Theorem~\ref{thm:ub} can be adapted to the infinite-horizon discounted MDP setting using the arguments of \citet{weisz2021query},
a more significant limitation of our lower bound (Theorem~\ref{thm:lb}) is that it does not immediately apply to the discounted setting, 
even though the lower bound of \citet{weisz2020exponential} (with exponentially many actions) does.
Indeed, the linear structure of our MDP construction relies on the rewards counting towards $v^\star$ and $q^\star$ with the same weight if they are delayed by a few steps.
The final limitation is that the results do not apply to problems where a controller is needed to work based on only features of the state. 
This is due to the limitation of the online planning protocol that requires state information to be passed to the planner (in addition to the features of the state).  
For settings like this, global planning, which directly aims for arriving at a policy that depends on the states only through the features, is more suitable.

\section{Notation and problem setup}\label{sec:prelim}

The purpose of this section is to introduce the notation we use and the necessary definitions that will allow us to precisely formulate the problems we study. %
We start with the notation. This is followed by a quick review of definitions and basic concepts concerning MDPs. The section will be closed by describing the planning problems considered.

\subsection{Notation}
Let $\N_+ = \{1, 2, \dots \}$ be the set of positive integers, and $\N=\{0\}\cup\N_+$.
Let $\bR$ denote the set of real numbers, $\B_d(r)=\{x\in\R^d\,:\,\norm{x}_2\le r\}$ the $d$-dimensional ball of radius $r$, and let $[i] = \{1,\dots,i\}$ be the set of integers from $1$ to $i$ for an integer $i \in \N_+$. For $i,j\in\N$, we use $[i:j]=\{i,i+1,\ldots,j\}$ if $i\le j$, and $[i:j]=\{\}$ otherwise.
For vectors $a$ and $b$ of compatible sizes, $\ip{a,b}=a^Tb$ denotes their inner product.
For a True or False statement $X$ (possibly depending on random variables),
let $\one{X}$ take $1$ if $X$ is True, and $0$ otherwise. %
Let $a \vmin b=\min(a,b)$ and $a \vmax b=\max(a,b)$.
For an event $E$, let $E^C$ denote its complementary event. 
Let $()$ denote the empty sequence.

\subsection{Episodic Markov decision processes with bounded rewards}
A \textbf{Markov decision process (MDP)} is defined by a tuple $M=(\cS,\cA,Q)$ of states, actions, and a transition-reward-kernel, respectively.
The structure $M$ defines a discrete time
sequential decision making problem where in time step $t=0,1,\dots$, 
an environment responds to an action $A_t(\in \cA)$ of an agent by transitioning 
from its current state $S_t(\in \cS)$ to a new random state $S_{t+1}(\in \cS)$ 
while also generating a random reward $R_{t+1}\in \R$ so that the distribution of $(R_{t+1},S_{t+1})$ given $S_0,A_0,R_1,S_1,\dots,A_{t-1},R_t,S_t,A_t$ 
is given by $Q(\,\cdot\,|\,S_t,A_t)$ regardless of the history before $S_t, A_t$.
Formally, $Q$ is a probability kernel from state-action pairs to reward-state pairs.
For simplicity, it is assumed that $R_{t+1}$ above is supported in $[0,1]$.

To simplify the presentation, 
we assume that the state space is finite. %
As noted beforehand, the definitions and results can be naturally translated to infinite state spaces.
Similarly, assume that the set of actions is finite and $\cA=[A]$ for some integer $A$.

In this work we focus on the \textbf{fixed-horizon undiscounted total expected reward objective}. 
Denoting the horizon by $H$,
under this objective, the goal is to find a \textbf{policy}, a way of choosing actions given the past, such that the total expected reward over $H$ steps is maximized regardless of the initial state of the process. (Formally, a policy is a stochastic kernel from histories to actions.)
The $H$ steps of the process is also called an \textbf{episode}.
As it is well known, the optimal policy, which maximizes the stated objective, depends on the number of steps left before the episode finishes. 
In this work, we will use an equivalent formulation which avoids this dependence.
In this formulation, only the first $H$ rewards can be non-zero, while the process continues indefinitely and the objective is changed to the total undiscounted expected reward. 
To emulate the fixed-horizon setting, one can then create $H$ disjoint copies of the state space, each corresponding to one step of the process while copying the transition structure to transition from one copy to the next one, and add an extra state ($\bot$) such that after $H$ steps this state is reached from which point this state is never left while the reward incurred remains zero regardless of the actions taken.
This is summarized below: %
\begin{assumption}
[Fixed-horizon MDP]
\label{ass:fixed}
The state space $\cS$ satisfies $\cS = \cup_{h=0}^{H} \cS_h$ with $\cS_{H} = \{ \bot \}$ and $Q$ is such that for any $s\in \cS_h$, $h\in [0:H-1]$ and $a\in \cA$, $Q(\cdot|s,a)$ is supported on
 $[0,1]\times ( \{\bot\}\cup\cS_{h+1} )$,
while for $h=H$, this support is $\{0\} \times \cS_{H} = \{0\}\times \{\bot\}$.
In particular, the sets $\{\cS_h\}_{h\in [0:H]}$ are pairwise disjoint.
\end{assumption}
Thanks to this assumption, when writing definitions, we consider the infinite horizon total expected reward criterion. This criterion assigns to a policy $\pi$ used in MDP $M$ from initial state $s\in \cS$ the value $v^\pi(s)$, which is defined as %
\begin{align}
v^\pi(s) = \E_{M,s}^{\pi} \left[ \textstyle \sum_{t=0}^\infty R_{t+1} \right]\,.
\label{eq:v-pi-def}
\end{align}
Here $\E_{M,s}^\pi$ is the expectation corresponding to the probability distribution $\bbP_{M,s}^\pi$ over trajectories of infinite length composed of state-action-reward triplets where this probability distribution arises from using policy $\pi$ in every step, with the first state fixed to $s$, while next states and rewards are generated according to $Q$. %
Under our assumption the value of any policy in any state is well-defined. The value, obviously, depends on $M$, but to minimize clutter this dependence is suppressed. We will also need the action-value function of a policy. This is defined similarly as above, except that one fixes both the initial state and the initial action.
Thus, for $(s,a)\in \cS \times \cA$, 
\begin{align}
q^\pi(s,a) = \E_{M,s,a}^{\pi} \left[ \textstyle \sum_{t=0}^\infty R_{t+1} \right]\,.
\label{eq:q-pi-def}
\end{align}
where 
$\E_{M,s,a}^\pi$ is the expectation corresponding to the probability distribution $\bbP_{M,s,a}^\pi$ over the trajectories as before, except that this time the first state-action pair is fixed to $(s,a)$ instead of just fixing the first state to $s$. 
Note that under Assumption~\ref{ass:fixed}, both $v^\pi$ (mapping states to reals, the \textbf{value function of $\pi$}) and $q^\pi$ (mapping state-action pairs to reals, the \textbf{action-value function of $\pi$}) are well defined and take values in $[0,H]$ and the infinite sums can be truncated after stage $H$.

Define $v^\star:\cS \to \R$ and $q^\star: \cS \times \cA \to \R$, 
the \textbf{optimal value} and, respectively, \textbf{optimal action-value function}
as
\begin{align}
v^\star(s) = \sup_{\pi} v^\pi(s), \qquad
q^\star(s,a) = \sup_{\pi} q^\pi(s), \qquad s\in \cS, a\in \cA\,. \label{eq:qstar-def}
\end{align}
A policy $\pi$ is said to be optimal if $v^\star = v^\pi$. It is well known that in our setting an optimal policy always exists and in fact the policy that uses any maximizer
of $q^\star(s,\cdot)$ when the state $S_t$ is $s$ is an optimal policy. 
This policy, as the choice of the action only depends on the last state, is called \textbf{memoryless}. 
Since the choice is also deterministic, the policy is also \textbf{deterministic}.
A deterministic memoryless policy can be concisely given as a map from states to actions. By slightly abusing notation, in what follows, we will identify such policies with such maps
and write $\pi:\cS \to \cA$ to denote a memoryless deterministic policy.
Given such a policy $\pi:\cS \to \cA$, its value functions $v^\pi$ and $q^\pi$ satisfy
the following equations:
\begin{align}
v^\pi(s)&=q^\pi(s,\pi(s))\,, \\ 
q^\pi(s,a)& = r(s,a) + \sum_{s'\in \cS} P(s'|s,a) v^\pi(s')\,, \qquad  s\in \cS, a\in \cA\,,
\label{eq:q-pi-bellman}
\end{align}
where $P(s'|s,a)$, derived from the transition kernel $Q$, is the probability of arriving at state $s'$ when the process is in state $s$ and action $a$ is taken
while $r(s,a)$ is the expected reward along this transition.
Formally, $P(s'|s,a) = Q([0,1]\times \{s'\}|s,a)$ and $r(s,a) = \int_{[0,1]\times \cS}  r dQ(r,s'|s,a)$.
The coupled equations \eqref{eq:q-pi-bellman} are
known as the Bellman equations for $\pi$ \citep{Put94}.%

Oftentimes in MDPs the rewards and the next states are independently chosen. In this case, $Q(\cdot|s,a)$ takes the form of the ``product'' of a probability kernel $R$ mapping state-action pairs to $[0,1]$ and the probability kernel $P$ mapping state-action pairs to states. In some constructions below, we will thus specify an MDP with the help of two such kernels. %

When the dependence of $v^\pi$, $v^\star$, or $q^\star$ on $M$ is important, we will put $M$ in the index of these symbols. For example, for a policy $\pi$ for $M$,
we will write $v_M^\pi$ to denote its value function in $M$.

\subsection{Online planning with featurized MDPs} %
\label{sec:planning}
The purpose of this section is to make the earlier definitions more formal. While we repeat some of the definitions (in a more precise way) from the introduction, to avoid unnecessary duplications, some definitions (e.g., the definitions of featurized MDPs) are not repeated.

As described in the introduction, in \textbf{online planning} 
a planner is used in a closed-loop configuration.
In our $H$-horizon setting, given an MDP $M=(\cS,\cA,Q)$ satisfying Assumption~\ref{ass:fixed}, 
in step $t\in [0:H-1]$ of using the planner,
the planner is given access to state $S_t$ of the process. By convention, $S_0\in \cS_0$ and thus we also have $S_t\in \cS_t$, for every $t$, where $(\cS_h)_{0\le h \le H}$ 
is the decomposition of $\cS$ from Assumption~\ref{ass:fixed}.
The planner is given access to a \textbf{simulation oracle} that can be queried
with state action pairs $(s,a)\in \cS \times \cA$, to which the oracle responds with a reward-state pair $(R',S')$ generated from $Q$:
\begin{align*}
(R',S')\sim Q(\cdot|s,a)\,.
\end{align*}
The planner is given the freedom to decide which queries and how many of them to use whenever it is called.
Eventually, the planner needs to stop querying and return an action $A_t\in \cA$, which is then used to generate the next state in $M$ and an associated reward:
\[
(R_{t+1},S_{t+1})\sim Q(\cdot|S_t,A_t)\,.
\] 
When used in an MDP $M$, a planner induces a policy $\pi_M$.
Note that 
$\pi_M$ is a stochastic policy (possibly history-dependent, if the planner saves information between its calls),
where the stochasticity comes from the randomness of the entire planner-oracle interaction (and possibly some independent randomization).
As such, $\pi$ itself is not a random element.
The goal of planner design is to find planners that induce near-optimal policies for large classes of MDPs while controlling the total computational cost of planning.
Formally, a \textbf{planner is $\delta$-sound} for class $\cM$ if for any $M=(\cup_h \cS_h,\cA,Q)\in \cM$
and any $s_0\in \cS_0$,
if the planner is used in $M$ while $S_0=s_0$, it holds that $\EE{ \sum_{t=1}^H R_t \,|\, S_0=s_0} \ge v_M^\star(s_0)-\delta$, or, equivalently, %
\begin{align*}
v_M^{\pi_M}(s_0)\ge v_M^\star(s_0)-\delta\,.
\end{align*}
Above, $\E$ is the expectation operator induced by the probability measure $\bbP$
induced
by the interconnection of the planner with the simulation oracle of $M$ and MDP $M$
over the interaction sequences that contain all information that flows between the planner, the simulation oracle and the MDP.
As all queries involve a computation step, a lower bound on the compute cost of planning is the \textbf{query cost}, which is defined as the expected number of queries that the planner uses in a planning step (or call).
For a planner and MDP, the maximal query cost over all possible calls to the planner is the worst-case query cost of using the planner on the MDP. The notion is naturally extended to the concept of query complexity of classes of MDPs by again taking the worst-case over the possible MDPs in the class.

In the special case when the planner is given access to the state space (and feature-map) of the MDP
and can thus call the simulator at any state of the MDP, the problem described so far is known as (online) planning with a \textbf{generative model}, or \textbf{global access}.
As opposed to this, if the planner is not given access to the state space (or the feature-map) and can only query the simulator at states that it has encountered beforehand, we say that the planner has \textbf{local access} only to the simulator.

When interacting with a featurized MDP $(M,\phiv)$ in the local access model, with, say, state features (i.e., $\phiv:\cS \to \R^d$, where $M= (\cS,\cA,Q)$), in step $t$ of the planning process, the planner is called with $(S_t,\phiv(S_t))$, while the simulation oracle returns
\begin{align*}
(R',S',\phiv(S'))\sim Q(\cdot|s,a)
\end{align*}
when queried with $(s,a)$.
In an alternate problem setting 
if the features are defined over state-action pairs, i.e., the planner interacts with a featurized MDP of the form $(M,\phiq)$ with $\phiq:\cS \times \cA \to \R^d$, 
 in step $t$ of the planning process,
the planner is called with $(S_t,\phiq(S_t),(\phiq(S_t,a))_{a\in \cA})$, while the simulation oracle returns
\begin{align*}
(R',S',(\phiq(S',a))_{a\in \cA})\sim Q(\cdot|s,a)
\end{align*}
when queried with $(s,a)$.
When the setting is such that both state and state-action features are available, the interaction is modified accordingly to include both types of features. 
The definitions of query cost and soundness remain the same for these settings.

\section{Lower bound (the proof of Theorem~\ref{thm:lb})}
\label{sec:lb}
In this section we give the proof of Theorem~\ref{thm:lb}. We start by introducing the high level ideas underlying the proof.

\subsection{Overview} %
\label{sec:proof-overview}
We prove the lower bound by designing a class of MDPs where by traversing the MDP, the agent effectively has to pick corners of a $p$-dimensional hypercube, in sequence, until either $K$ picks were made or a pick was sufficiently close to the secret ``solution'' corner 
$\stheta^\star$.
Here, $p\approx H^{1/2}\wedge d^{1/4}$ (large if both $H$ and $d$ are large) and $K\approx H/p$ (large if $H$ is large).
If the agent picks a corner close to the solution, the episode is effectively terminated and the agent receives the highest possible reward achievable from that state.
Otherwise, the agent's next pick has to substantially differ from the previously picked corner.
After each choice, the highest reward achievable shrinks by a penalty factor that is governed by how different the subsequent picks are: picking dissimilar corners results in a larger penalty (i.e., a smaller penalty factor).
Since subsequent picks need to be substantially different, this means that $q^\star$ (or $v^\star$) reduces at an exponential rate throughout the episode until a guess is sufficiently close to the solution or all $K$ picks are exhausted, in which case the agent receives a Bernoulli reward with expectation $\exp(-\Omega(K))$. 
Without additional information, guessing sufficiently close to the solution is a needle-in-a-haystack problem with an exponentially large haystack: with probability above (say) $3/4$,
the secret corner will not be found within $\exp(\Omega(p))$ guesses.
Additional information is not provided to the agent as long as the final reward is 0. 
Since the probability that this Bernoulli outcome is identically zero
for the first $\exp(\Omega(K))$ guesses can be made to be $3/4$ or larger,
if a planner uses at most 
$\exp(\Omega(p \vmin K))$ guesses,
with probability at least $1/2$, neither blind guessing nor the Bernoulli outcomes will lead to success.
Thus, in expectation, any sound planner has to query more than $\exp(\Omega(p \vmin K))$ times.

To achieve realizability of $q^\star$ (or $v^\star$), it is sufficient if the value of the optimal policy is a low-order polynomial of the $p$-dimensional secret solution at any state in the MDP. 
To achieve this, the mechanics of choosing a guess and the penalty factor are carefully chosen in such a way that the optimal policy has a simple ``greedy'' structure that moves any guess as close as possible to the solution.
The value of this greedy optimal policy is then proved to be a $4^{\text{th}}$-order polynomial of 
$\stheta^\star$, which gives rise to a $d\approx p^4$ dimensional feature-map that can realize the optimal values.

For the sake of simplicity and modularity, rather than defining the MDP, we first define 
a simplified ``abstract game'' where an ``abstract planner'' has to guess the above-mentioned 
secret parameter. This abstract game is essentially what has been described in the previous paragraph.
This construction focuses on the information theoretic aspect of the proof, leaving the construction of the MDP with the required realizability properties to the subsequent sections.

\subsection{Abstract game}
\label{sec:abstract-game}

The abstract game has a length parameter $K\in\N_+$ and 
an integer dimensionality 
parameter $\sd\ge2$, which are known to the abstract planner.
Let $W=\{-1,1\}^\sd$.
Let $\bm{0}$ and $\bm{1}$ indicate the $\sd$-dimensional vectors of all zeros and all ones, respectively.
For vectors $x$ and $y$ from $W$, define $\diff(x,y)$ as the Hamming distance between $x$ and $y$, i.e., the number of components where $x$ and $y$ are different. We will use the property of the Hamming distance that it can be written as an (affine) bilinear function of its arguments:
for $\stheta_1,\stheta_2 \in W$,
\begin{align}
\label{eq:diff-def}
\diff(\stheta_1,\stheta_2)=\frac12\left(\sd-\ip{\stheta_1,\stheta_2}\right)\,.
\end{align}
Note that $\diff(\cdot,\cdot)$ is a metric on the set $W$.
Let 
\begin{align}
W^\star=\{w\in W\,:\, \sd/4\le \diff(\bm{1}, \stheta)\le 3\sd/4\} \label{eq:thetastar-distance-to-1}
\end{align}
be the set that will hold the game's secret parameter: $\stheta^\star\in W^\star$. 
For any $k\in\N$, let 
\begin{align}
W^{\circ k}=\{(\stheta_i)_{i\in[k]} \in W^{k}\,:\, \diff(\stheta_{i-1},\stheta_{i})\ge \sd/4 \text{ for } i\in [k]\}\,,
\label{eq:w-circ}
\end{align}
with $\stheta_0:=\bm{1}$ defined for convenience,
be the subset of $k$-length sequences of $W$ where the elements are ``sufficiently far'' from each other. %

The union of these over $k\le K$ is the action set of the bandit-like game.
Given $\stheta^\star$,
the reward function
$f_{\stheta^\star}:\{()\}\cup\bigcup_{k\in[K]} W^{\circ k}\to\R$ (index dropped when clear from context) is defined as follows (again $\stheta_0:=\bm{1}$):%
\footnote{The reason for this form of $f$ will become clear only when the MDP corresponding to the abstract game is defined. For now, let us only note that (1) as the input sequence grows in size, their elements being sufficiently far ensures an exponential rate of reduction of $f$, and (2) $g(x)$ is the second-order Taylor expansion of $(1-1/\sd)^{x}$, which ensures through some inequalities that the optimal strategy for maximizing $f$ is to greedily move towards $\stheta^\star$ in the MDP as fast as possible. A simple optimal policy with a low-order polynomial expression for $f$ allows 
deriving linear features for the MDP's value function.} 
$f(())=g(\diff( \stheta_{0}, \stheta^\star))$, and for $k\in[K]$,
\begin{align*}%
f_{\stheta^\star}\left((\stheta_{i})_{i\in[k]}\right)
&=
\left(\prod_{i\in[k]} g(\diff( \stheta_{i-1},  \stheta_{i}))\right) g(\diff( \stheta_{k}, \stheta^\star))
 & \text{where}\\
g(x)&=1-\frac{x}{\sd}+\frac{(x-1)x}{2\sd^2}\,. %
\end{align*}

The game is sequential. It proceeds in steps where the abstract planner performs a query and receives a corresponding response (both the query and the response may be randomized).
At each step $t\in\N_+$, the abstract planner randomly chooses 
whether to continue or not, and what its output or next query (correspondingly) is.
If it continues, it chooses a sequence length $L_t\in[K]$,
and a sequence $S_t=(\stheta^t_i)_{i\in[L_t]}\in W^{\circ L_t}$.
Otherwise, if it returns, it chooses its output
$S_t=(\stheta^t_i)_{i\in[8]}\in W^{\circ 8}$. Note that the output is confined to have the fixed length of $8$.%
\footnote{The constant $8$ here is sufficiently small to prove that planners cannot guess close enough to $\stheta^\star$ with any of these $8$ attempts, yet large enough so that to achieve a small suboptimality in the MDP problem (that will be derived later), it will be crucial to guess a vector among these $8$ vectors that is close to $\stheta^\star$.}
To distinguish this from the case when the planner continues, we let $L_t=0$ denote that the planner wants to return an output.
Let $N=\min \left\{t\in{\N_+} \,:\, L_t = 0 \right\}$ indicate the step at which the planner returns. Thus, the planner's output is $S_{N}$.

At step $t$, denote the choice of the planner by $X_t=(L_t, S_t)$.
If the planner is not done yet ($L_t>0$, and thus $t<N$) then, in response to the planner's query,
a random response $Y_t=(U_t, V_t, Z_t)\in \{0,1\}\times\{0,1\}\times [0,1]$ is 
generated 
as follows:
\begin{itemize}
\item $U_t$ indicates whether the penultimate component of $S_t$ is close to $\stheta^\star$ (for convenience define $\stheta^t_0=\bm{1}$):
\[U_t = \one{\diff(\stheta^t_{L_t-1}, \stheta^\star)<\sd/4}\,.\]
\item $V_t$ indicates whether the last component of $S_t$ is close to $\stheta^\star$:
\[V_t=\one{\diff(\stheta^t_{L_t}, \stheta^\star)<\sd/4}\,.\]
\item $Z_t$ is distributed as $\Bernoulli(f_{\stheta^\star}(S_t))$ if either $V_t=1$ (the last component of $S_t$ is close to $\stheta^\star$) or $L_t=K$ (all components are used in $S_t$), else $Z_t=0$.
Here, $\Bernoulli$ denotes the Bernoulli distribution. 
This is well-defined as $f_{\stheta^\star}(S_t)\in[0,1]$ by Lemma~\ref{lem:f-bounds}.
\end{itemize}

If, on the other hand, the planner indicates that it is done ($L_t=0$, and thus $t=N$) then there is no feedback, but the payoff (reward) to the planner is
\begin{align}\label{eq:final-abstract-reward}
R=f_{\stheta^\star}\left((\stheta^N_i)_{i\in [k^\star]}\right)
\end{align}
where $k^\star=k^\star(S_t;\stheta^\star)$
denotes the first component of $S_t=(\stheta^N_k)_{k\in [8]}$ that is sufficiently close to $\stheta^\star$, or $8$ if none of them are:
\[
k^\star=\min\{k\in [8]\,:\, k=8\text{ or }\diff(\stheta^N_k, \stheta^\star)<\sd/4\}\,.
\]
For future reference, it will be useful to introduce $\tau_{\stheta^\star}(s)$ 
to denote the first $k^\star(s,\stheta^\star)$ components of $s=(\stheta_i)_{i\in [8]}$ so that 
$R = f_{\stheta^\star}(\tau_{\stheta^\star}(S_t))$.
While the interaction is over at this stage, for simplifying notation, we introduce $Y_t$ and define it as  $Y_t=(0,0,0)$.

This finishes the description of the abstract game; for a given value of $\stheta^{\star}$ we will refer it as ``abstract game $\stheta^{\star}$''. To summarize, in this game, the planner can choose actions from a combinatorially structured action set to collect information for the final round where it needs to choose an action from a smaller (but still combinatorially large) subset of the action set. The feedback is nonlinear. The essence of the information theoretic argument that will follow will be that good planners essentially need to find $\stheta^\star$.

For these information theoretic arguments, as well as the statement of the main result of this section, some extra definitions are necessary.
For $t\in\N_+$, let $F_t=(X_i, Y_i)_{i\in[t-1]}$. For each step $t$ sequentially, if the game is not over yet, i.e., $t-1<N$, the planner $\aplan$ defines the distribution of $X_t$ given $F_t$. Given $F_t$ and $X_t$, the distribution of $Y_t$ is defined as above.
Together, $\aplan$ and $\stheta^\star$ define $\bbP_{\stheta^\star}^\aplan$, the probability distribution over interaction sequences $(X_t, Y_t)_{t\in[N]}$ between the planner and the game, where the sequence needs to satisfy that $L_t>0$ for $t<N$ and $L_N=0$.%
\footnote{Luckily for us, $F_t$ takes values in a finite set, which makes it trivial to show that 
$\bbP_{\stheta^\star}^\aplan$ with the required properties exist.}
The planner is well-defined if $\bbP_{\stheta^\star}^\aplan [N<\infty]=1$.
Let $\E_{\stheta^\star}^\aplan$ be the expectation operator corresponding to $\bbP_{\stheta^\star}^\aplan$.
The abstract planner is sound with worst-case query cost $\bar N$ if for all $\stheta^\star \in W^\star$,
$\E_{\stheta^\star}^\aplan [N-1]\le \bar N$, and $\E_{\stheta^\star}^\aplan [R] \ge 
\max_{s\in \W^{\circ 8}} f_{\stheta^\star}(\tau_{\stheta^\star}(s))-0.01$.
We note in passing that 
$\max_{s\in W^{\circ 8}} f_{\stheta^\star}(\tau_{\stheta^\star}(s))
=f_{\stheta^\star}(())$, i.e., the maximizing sequence is the empty sequence.

The main result of this section is the following claim, which states that the abstract game is hard:
\begin{theorem}\label{thm:abstract-lb}
For any abstract planner that is sound with query cost $\bar N$,
\[
\bar N = 2^{\Omega\left(\sd\wedge K\right)}\,.
\]
\end{theorem}
The proof is given in a number of lemmas. We start with some elementary properties of $f_{\stheta^\star}$: %
\begin{lemma}[Properties of $f_{\stheta^\star}$]
\label{lem:f-bounds}
For any $\stheta^\star\in W^\star$, $k\in \N_+$, $s= (\stheta_{k'})_{k'\in [k]}\in W^{\circ k}$,
the following hold:
\begin{align}
 \frac{11}{32} \le f_{\stheta^\star}(()) &\le \frac{25}{32}\,, \label{eq:f-lb}\\
0<f_{\stheta^\star}(s) &\le \left(\frac{25}{32}\right)^{k+\one{\diff(\stheta_k,\stheta^\star)\ge \sd/4}}  \label{eq:f-ub-last-far}
\end{align}
\end{lemma}
\begin{proof}
We prove Eq.~\ref{eq:f-ub-last-far} by first showing that  %
\begin{align}
0<f_{\stheta^\star}(s) &\le \left(\frac{25}{32}\right)^k.
\label{eq:f-ub}
\end{align}
This follows since $f$ is the product of $k+1$ terms, each defined using the function $g$.
Now, notice that $g(x)$
decreases as $x$ increases in the range $0\le x \le \sd$, so for all $k'\in[k]$, 
thanks to $\diff( \stheta_{k'-1},  \stheta_{k'})\ge \sd/4$ which holds since by assumption $s\in W^{\circ k}$, we have
\[
0< g(\sd) \le g(\diff( \stheta_{k'-1},  \stheta_{k'})) \le g(\sd/4) < \frac{25}{32} \,.
\]
This, together with $0<g(0)\le 1$ proves Eq.~\ref{eq:f-ub}.
To finish the proof of Eq.~\ref{eq:f-ub-last-far}, note that if
$\diff(\stheta_k,\stheta^\star)\ge \sd/4$
then, similarly to the previous case, 
we have $0<g(\diff( \stheta_{k},  \stheta^\star)) \le g(\sd/4) < \frac{25}{32}$, 
which implies Eq.~\ref{eq:f-ub-last-far}.
As $\stheta^\star\in W^\star$, $\frac14\sd\le \diff(\bm{1}, \stheta^\star)\le \frac34\sd$. Hence,
$f_{\stheta^\star}(()) = g(\diff(\bm{1}, \stheta^\star))\ge g(\frac34\sd)\ge\frac{11}{32}$ and
$f_{\stheta^\star}(()) \le g(\frac14\sd)\le\frac{25}{32}$.
\end{proof}

Let
\begin{align}\label{eq:n-def}
n=\floor{\min\left(\frac{e^{{\frac{\sd}{8}}}}{16} - 5, \frac{\frac1\epsilon-1}{7.5}\right)}\,,
\end{align}
where
\begin{align*}
\epsilon&=\left(\frac{25}{32}\right)^{K+1}\,.
\end{align*}
For any $\stheta^\star\in W^\star$, let $E_n^{\stheta^\star}$ be the event
when in the first $n$ steps the planner does not hit on any vector that is close to $\stheta^\star$: %
\begin{align*}
E_n^{\stheta^\star} &= 
\bigcap_{t\in [n]}
\Bigg\{ t>N \text{ or } \left(t=N \text{ and } \min_{i\in[8]}\diff(\stheta_i^N,\stheta^\star)\ge \frac{\sd}{4}\right)\\
&\qquad\qquad\qquad \text{ or } 
\left(
t<N \text{ and }
\diff(\stheta^t_{L_t-1}, \stheta^\star)\ge\sd/4 \text{ and }
\diff(\stheta^t_{L_t}, \stheta^\star)\ge \sd/4
\right)
\Bigg\}\,.
\end{align*}
We define the ``abstract game $0$'' (and, for any planner $\aplan$, the associated probability distribution $\bbP_0^\aplan$) to be a variant of the game  where the responses are $Y_t \equiv(0,0,0)$ for all $t\in\N_+$ (irrespective of the choices of the planner).

Our next lemma claims that the bad event $E_n^{\stheta^\star}$ happens with large probability in abstract game $\stheta^\star$ whenever it happens with large probability in abstract game $0$. The reason for this is that the probability of ever receiving nonzero feedback on the bad event is a small value, which in fact can be bounded by $\epsilon$ (the only way to receive nonzero feedback is by playing to the end, hence $\epsilon$ appears). 
From here it will follow that since the number of steps is at most $n$ (bad events are defined for interactions of length at most $n$),
the probability of $E_n^{\stheta^\star}$ in game $\stheta^\star$ is at least the probability of this event in game $0$ times $(1-\epsilon)^n$, and the latter is
lower bounded by an absolute constant because $n$ is chosen to be not too large compared to $1/\epsilon$.

\begin{lemma}\label{lem:7-8-m0-to-mtheta}
Take $n$ as defined in Eq.~\ref{eq:n-def}.
Then, for any abstract planner $\aplan$ and for any 
$\stheta^\star \in W$,
\begin{align*}
\bbP_{\stheta^\star}^\aplan(E_n^{\stheta^\star}) 
&\ge 
\frac78 \bbP_0^\aplan(E_n^{\stheta^\star})\,.
\end{align*}
\end{lemma}
\begin{proof}
We prove that
\begin{align}\label{eq:lb-abstract-transfer}
\bbP_{\stheta^\star}^\aplan(E_n^{\stheta^\star}) &\ge \left(1-\epsilon\right)^n \bbP_0^\aplan(E_n^{\stheta^\star})	\,.
\end{align}
Since by its choice, $n$ satisfies $n\le \left(\frac1\epsilon-1\right)/7.5$, or, equivalently, $1-\epsilon \ge 1- \frac{1}{1+7.5 n}$, it follows that 
\begin{align*}
(1-\epsilon)^n \ge \left(1- \frac{1}{1+7.5 n}\right)^n \ge \lim_{n\to\infty}  \left(1- \frac{1}{1+7.5 n}\right)^n = e^{-1/7.5}> 7/8\,,
\end{align*}
which shows that it suffices to prove Eq.~\ref{eq:lb-abstract-transfer}.

Let $(X_t,Y_t)_{t\in [N]}$ be a complete interaction history and
let $H$ denote the first $n\wedge N$ components of  this (thus, $H$ is shorter than the complete sequence when $n < N$). Let
 $\cH$ be the set of all possible values that $H$ can take.
For $h\in \cH$, let $E_h = E_n^{\stheta^\star} \cap \{ H= h\}$. Clearly, $E_n^{\stheta^\star}$ is the disjoint union of the sets $\{E_h\}_{h\in \cH}$. 
Let $\cH^+=\{ h\in \cH \,:\, \bbP_{0}^\aplan(E_h)>0\}$.
Then, $\bbP_0^\aplan(E_n^{\stheta^\star})=\sum_{h\in\cH^+}\bbP_0^\aplan(E_h)$, and we prove Eq.~\ref{eq:lb-abstract-transfer} by showing that for any $h\in \cH$,
\begin{align}\label{eq:to-show-rho-ub}
\rho = \frac{\bbP_{\stheta^\star}^\aplan \left(E_h\right)}{\bbP_{0}^\aplan \left(E_h\right)} \ge \left(1-\epsilon\right)^n\,.
\end{align}
Fix $h\in \cH^+$ and let $h = (x_t,y_t)_{t\in [n']}$ for some $0<n'\le n$.
Further, let $x_t = (l_t,s_t)$. Note that for $t<n'$, $l_t>0$ and either $n'=n$ or $l_{n'}=0$.

As $\bbP_{0}^\aplan(E_h)>0$, $y_t=(0,0,0)$ for all $t\in[n']$.
By definition of $\bbP_{\stheta^\star}^\aplan$ and $\bbP_{0}^\aplan$, both the numerator and denominator factorizes into the product of $n'$ terms.
Given the same history, the distribution of $X_t$ under both 
$\bbP_{\stheta^\star}^\aplan$ and $\bbP_{0}^\aplan$ are identical, 
so the terms that do not cancel remain:
\[
\rho=\prod_{t=1}^{n'} \frac{\bbP_{\stheta^\star}^\aplan\left(Y_t=(0,0,0) \,|\, X_t=x_t\right)}{\bbP_{0}^\aplan\left(Y_t=(0,0,0) \,|\, X_t=x_t\right)} 
=
\prod_{t=1}^{n'} \bbP_{\stheta^\star}^\aplan\left(Z_t=0 \,|\, X_t=x_t\right)\,,
\]
where $Y_t=(U_t, V_t, Z_t)$. Here, the last equality follows since
 $\bbP_{0}^\aplan[Y_t=(0,0,0) \,|\, X_t=x_t]=1$ by definition and 
 $
\bbP_{\stheta^\star}^\aplan\left(Y_t=(0,0,0) \,|\, X_t=x_t\right)
 =
\bbP_{\stheta^\star}^\aplan\left(Z_t=0 \,|\, X_t=x_t\right)
 $
 because on $E_h \subset E_n^{\stheta^\star}$, $U_t=V_t=0$ holds $\bbP_{\stheta^\star}^\aplan$ almost surely.
Now, by definition,
$\bbP_{\stheta^\star}^\aplan\left(Z_t=1 \,|\, X_t=x_t\right)
=
f_{\stheta^\star}(s_t) \one{ \diff(\stheta_{l_t}^t,\stheta^\star)\le p/4 \text{ or } l_t=K }$.
Since $E_h\subset E_n^{\stheta^\star}$, 
$\diff(\stheta_{l_t}^t,\stheta^\star)\le p/4 $ does not hold. Hence,
$\bbP_{\stheta^\star}^\aplan\left(Z_t=1 \,|\, X_t=x_t\right)
=
f_{\stheta^\star}(s_t) \one{ l_t=K } \le (25/32)^{K+1}=\epsilon$, 
where the inequality 
follows from Lemma~\ref{lem:f-bounds} using again that $E_h\subset E_n^{\stheta^\star}$ and thus 
the last component of $s_t$ must be ``far'' from $\stheta^\star$.
Putting things together and using that $n'\le n$ gives that $\rho\ge (1-\epsilon)^n$, as required.
\end{proof}

We plan to argue that the bad event happens with large probability in game $0$.
In this game, by definition, the planner needs to guess $\stheta^\star$ blindly (as there is no feedback ever). Hence, the success of the planner depends on whether they can without any feedback stumble upon $\stheta^\star$. To bound this success rate, it will be useful to bound the number of vectors close to a given vector in the hypercube $W$: %

\begin{lemma}\label{lem:close-to-corner-ct}
For any $\tilde\stheta\in W$, let $\Wc(\tilde\stheta)=\{\stheta\in W \,|\, \diff(\stheta, \tilde\stheta)<\sd/4\}$.
Then,
\[
|\Wc(\tilde\stheta)| \le 2^\sd \exp\left(-\frac{\sd}{8}\right)
\]
\end{lemma}
\begin{proof}
By symmetry of the $\sd$-dimensional hypercube, without loss of generality, let $\tilde\stheta=\bm{1}$ and $\Wc=\Wc(\tilde\stheta)$. 
Let $X=(X_i)_i\in W$ be a uniformly distributed random variable on $W$. Note that the components $X_i$ of $X$ are independent Rademacher random variables.
We have
\begin{align*}
|\Wc|  & =
\sum_{\stheta\in W} \one{\ip{\stheta,\bm{1}}>\sgamma} =
|W|\, \bbP(\ip{X, \bm{1}}>\sgamma) \\
& = 2^\sd\, \bbP \left(\sum_{i\in\sd}X_i>\sgamma\right) \,
\le 2^\sd \exp\left(\frac{-2(\sd/2)^2}{4\sd}\right)  = 2^{\sd}\exp\left(-\frac{\sd}{8}\right)\,,
\end{align*}
where the second inequality holds by Hoeffding's inequality.
\end{proof}

Our next lemma shows that for any planner
the probability of a bad event has an absolute lower bound. We use the previous lemma to show that for any planner there exists a $\stheta^\star$ such that the probability of the corresponding bad event is lower bounded in game $0$, and then we apply Lemma~\ref{lem:7-8-m0-to-mtheta} to get a lower bound for the same event in game $\stheta^\star$.

\begin{lemma}\label{lem:no-reveal-whp}
For any abstract planner $\aplan$
there exists $\stheta^\star\in W^\star$
such that
\[
\bbP_{\stheta^\star}^\aplan(E_n^{\stheta^\star})  \ge \left(\frac78\right)^2\,.
\]
\end{lemma}
\begin{proof}
For any $\hat \stheta\in W^\star$, under event $\left(E_n^{\hat\stheta}\right)^c$,
either
there exists $t\in[n\vmin (N-1)]$ such that
$\diff(\stheta^t_{L_t-1}, \hat\stheta)<\sd/4$ or $\diff(\stheta^t_{L_t}, \hat\stheta)< \sd/4$,
or for some $i\in[8]$, 
$\diff(\stheta^N_i, \hat\stheta)<\sd/4$. That is,
$\left(E_n^{\hat\stheta}\right)^c \subset \{ \hat \stheta \in Z \}$ where
\[
Z :=
\bigcup_{t\in[n\vmin (N-1)]} \left(\Wc(\stheta^t_{L_t-1}) \cup \Wc(\stheta^t_{L_t} )  \right)
\bigcup 
\left(
\bigcup_{i\in[8]} \Wc(\stheta^N_i) \right) \,.
\]
By Lemma~\ref{lem:close-to-corner-ct}, 
\begin{align}
|Z| \le  (2n+8) 2^\sd \exp\left(-\frac{\sd}{8}\right)\,.
\label{eq:cardbound}
\end{align}
We also have that
$W^\star=W\setminus \Wc(\bm{1})\setminus \Wc(-\bm{1})$, so
 $|W^\star|\ge 2^\sd \left(1-2\exp\left(-\frac{\sd}{8}\right)\right)$.
As $\stheta^\star\in Z$ is the good event for the planner, 
we define
\begin{align}\label{eq:thetastar-choice}
\stheta^\star = \argmin_{\hat\stheta\in W^\star} \bbP_{0}^\aplan \left(\hat\stheta\in Z\right)\,.
\end{align}
Putting things together and using that $Z\subseteq W$, we get
\begin{align*}
\MoveEqLeft 
2^\sd \left(1-2\exp\left(-\frac{\sd}{8}\right)\right)
\bbP_{0}^\aplan \left(\stheta^\star \in Z\right)
\le 
|W^\star| 
\bbP_{0}^\aplan \left(\stheta^\star \in Z\right) \\
& \le
\sum_{\hat\stheta\in W^\star} \bbP_{0}^\aplan \left(\hat\stheta\in Z\right) 
 \le
\sum_{\hat\stheta\in W} \bbP_{0}^\aplan \left(\hat\stheta\in Z\right) 
 =
\E_{0}^\aplan [ \left|Z\right| ] 
\le (2n+8) 2^\sd \exp\left(-\frac{\sd}{8}\right)\,,
\end{align*}
Rearranging and using $(E_n^{\stheta^\star})^c \subset \{\stheta^\star\in Z\}$, we get
\[
\bbP_{0}^\aplan \left(\left(E_n^{\stheta^\star}\right)^c\right) \le
\bbP_{0}^\aplan \left(\stheta^\star \in Z\right) \le 
\frac{(2n+8) 2^\sd \exp\left(-\frac{\sd}{8}\right)}{2^\sd \left(1-2\exp\left(-\frac{\sd}{8}\right)\right)}
\le 2(n+5)\exp\left(-\frac{\sd}{8}\right)
\le \frac{1}{8}\,,
\]
where the last two inequalities follow by our choice of $n$.
Combining this with Lemma~\ref{lem:7-8-m0-to-mtheta} finishes the proof.
\end{proof}

With this, we are ready to prove Theorem~\ref{thm:abstract-lb}. In fact, all that is left to show is that 
if the planner is sound, then the probability of the bad event cannot be too high. That is, connecting the bad event to poor performance.

\begin{proof}[Proof of Theorem~\ref{thm:abstract-lb}]
Take a sound abstract planner $\aplan$ with query cost $\bar N$.
Let $\stheta^\star$ be the vector whose existence is guaranteed by the previous lemma.
By Markov's inequality,
\[
	\bbP_{\stheta^\star}^\aplan \left[ N-1 \ge n \right] \le \frac1n \bar N\,.
\]
Let $E'$ be the event under which both $N-1< n$ and $E_n^{\stheta^\star}$ hold:
$E' = \{ N-1<n \} \cap E_n^{\stheta^\star}$.
 By the union bound and Lemma~\ref{lem:no-reveal-whp},
\begin{align}\label{eq:e-prime-probab}
\bbP_{\stheta^\star}^\aplan\left[E'\right] \ge \left(\frac78\right)^2 - \frac1n \bar N\,.	
\end{align}
Under the event $E'$, the output of the planner $(\stheta^N_i)_{i\in[8]}$ satisfies 
$\diff(\stheta^N_i, \stheta^\star)\ge \sd/4$ for $i\in[8]$, and therefore $k^\star=8$ and,
by  Lemma~\ref{lem:f-bounds},  the reward $R$ of the game satisfies
$R< \left(\frac{25}{32}\right)^9$.
Therefore, combined with the soundness of $\aplan$, we get
\begin{align*}
\frac{11}{32}-0.01 
 \le f_{\stheta^\star}(())-0.01 \le \E_{\stheta^\star}^\aplan\left[R\right]
& \le 
\left(\frac{25}{32}\right)^9
+(1-\bbP_{\stheta^\star}^\aplan\left[E'\right] ) \frac{25}{32} \\
& \le
\left(\frac{25}{32}\right)^9
+\left(1-\left(\frac78\right)^2 \right) \frac{25}{32} + \frac{\bar N}{n} \frac{25}{32}\,,
\end{align*}
where we used Lemma~\ref{lem:f-bounds} to bound $f_{\stheta^\star}(())$, and the maximum value of $R$ (maximum value of $f$) by $\frac{25}{32}$.
To satisfy this inequality, we must have $\bar N>0.05 n$, and thus
by substituting Eqs.~\ref{eq:n-def} and simplifying we get
\begin{align*}
	\bar N &=
	\Omega\left(\min\left(\frac{e^{{\frac{\sd}{8}}}}{16} - 5, \frac{\frac1\epsilon-1}{7.5}\right)\right)
	= \min\left(2^{\Omega(\sd)},\Omega\left(\left(\frac{32}{25}\right)^{K+1}\right)\right)
	=2^{\Omega(\sd \wedge K) }\,.
\end{align*}
\end{proof}

\subsection{Description of the hard MDP class}\label{sec:mdp-summary}
Given a large enough horizon $H$ and a large enough dimension $\ld$,
in this section we construct a class of featurized MDPs with horizon $H$ and feature-space
dimension $\ld$,
such that 
{\em (i)}
each featurized MDP in the class corresponds to an abstract game with parameters $(K,\sd)$
such that $H \approx K\sd$, $A=\sd \approx \ld^{1/4}\wedge H^{1/2}$
{\em (ii)}
each MDP $M_{\stheta^\star}$ is associated with some abstract game 
$\stheta^\star\in W^\star \subset W = \{-1,1\}^p$;
{\em (iii)}
the feature-maps associated with the MDPs do not depend on $\stheta^\star$;
{\em (iv)} the respective realizability assumptions are satisfied by the featurized MDPs in the class;
{\em (v)} a planner that is guaranteed to achieve a high value in the MDPs 
can be used to achieve high values in the associated abstract game, which also means that
{\em (vi)} for every $\stheta^\star\in W^\star$, 
one should be able to emulate the queries in the featurized MDP  associated with $\stheta^\star$
using queries that are available in the abstract game with $\stheta^\star$, while the MDP planner should not get any information about $\stheta^\star$ by any other means than through these queries.

In the abstract game, at the end the planner needs to choose a sequence 
$(\stheta_i)_{i \in [8]}\in W^{\circ 8}$.
This will correspond to the first $8p$ steps of the path that the MDP planner traverses in the MDP,
which will have deterministic dynamics. 
To guarantee that the number of actions is small, 
choosing such a weight sequence will be implemented in the MDP by first choosing $\stheta_1$ in $p$ steps, then choosing $\stheta_2$ in another $p$ steps, etc.
In each of the $\sd$ steps of these rounds, choosing an action $a\in [p]$ will allow 
the MDP planner to flip component $a$ of the weight associated with the round. 
In particular, in the first $\sd$ steps, the components of $\stheta_1$ are chosen this way, starting from the weight vector $\stheta_0 = \bm{1}$. 
In the next $\sd$ steps, the components of $\stheta_2$ are chosen this way, but this time starting with $\stheta_1$. 
The process is identical for choosing $\stheta_{k}$ based on $\stheta_{k-1}$, where we let $1\le k \le K$ go up to $K$ to support arbitrary queries in the abstract game. 
To guarantee that the path chosen is in $\cup_k W^{\circ k}$, further rules are necessary. 
In particular, since we need to guarantee that $w_k$ differs from $w_{k-1}$ by at least $\sd/4$ positions, the dynamics is chosen so that in the first $\lceil \sd/4\rceil$ steps within the $k$th round, 
if an action is repeated 
then it is called illegal, and leads to the end-state $\bot$, while in the remaining $p-\lceil \sd/4\rceil\approx 3\sd/4$ steps an action repeat is called legal and leads to a ``frozen'' weight,
i.e., starting from the first such repeated action the weight associated with the path cannot be changed until the round is over.
These rules guarantee that if a path of length $k\sd$ does not end up in $\bot$, the path uniquely determines an element of $W^{\circ k}$ 
(in fact, the last state alone uniquely determines such an element). 
We associate with every action sequence subject to the constraints just described a unique state, which can be seen as a node on the action tree. 
We will say that a state $s\ne \bot$ belongs to some round $k\in[0:K-1]$, 
if the length $l$ of the associated action sequence $(a_i)_{i\le l}$ 
satisfies $k\sd \le l < (k+1)\sd$.
We say that the state is in step $i$ of round $k$ if also $l=k\sd+i$.

Normally, the transitions of the MDP follow the path in the action tree just described, and the rewards are zero.
However, there are two exceptions that depend on $\stheta^\star$.
To describe them, note that a state $s\ne \bot$ that is in step $\sd-1$ of some round $k$ is one step away from finalizing the choice of weight vector $\stheta_{k+1}$.
Indeed, such a state, together with the action performed in that state, defines the weight sequence 
$(\stheta_i)_{i\in [k+1]}\in W^{\circ k+1}$, while a shorter sequence $(\stheta_i)_{i\in [k]}\in W^{\circ k}$ is defined by all states $s\ne \bot$ that are in any step $i$ of some round $k$.

For a state that is in some step $i$ of some round $k$, the aforementioned exceptions to the MDP dynamics are:
{\em (i)} if $k>0$ and $\diff(\stheta_{k},\stheta^\star)<p/4$;
{\em (ii)} else if $i=\sd-1$, and either $k=K-1$ (last step of episode), or
$\diff(\stheta_{k+1},\stheta^\star)<p/4$.
In case {\em (i)}, the next state is $\bot$, and 
the reward is deterministically set to $\ip{\phi,\ltheta^\star}$, where $\phi$ is the feature-vector associated with the state or the state-action pair (depending on which class of featurized MDPs are considered), and $\ltheta^\star$ is a hidden weight vector corresponding to $\stheta^\star$.
In case {\em (ii)},
a Bernoulli reward with parameter 
$f_{\stheta^\star}( (\stheta_i)_{i\in [k+1]} )$ is generated, while also transitioning to $\bot$.
Note that the states associated with case {\em (i)} are unreachable from the initial state
 as any path to such state goes through a state that satisfies {\em (ii)}. 
While there is much information to be gained from any query where the state is of this type,
planners with local access can never issue such queries, 
while planners with global access still have very little chance of encountering such a state (the proportion of these states is exponentially small as can be seen from, e.g., the result of Lemma~\ref{lem:close-to-corner-ct}).
We refer the reader to Figure~\ref{fig:mdp-illustration} for an illustration of the MDP dynamics and the associated reward structure, and to Eq.~\ref{eq:rs-def} for a more precise definition.

\begin{figure}[t]
\includegraphics[width=\textwidth]{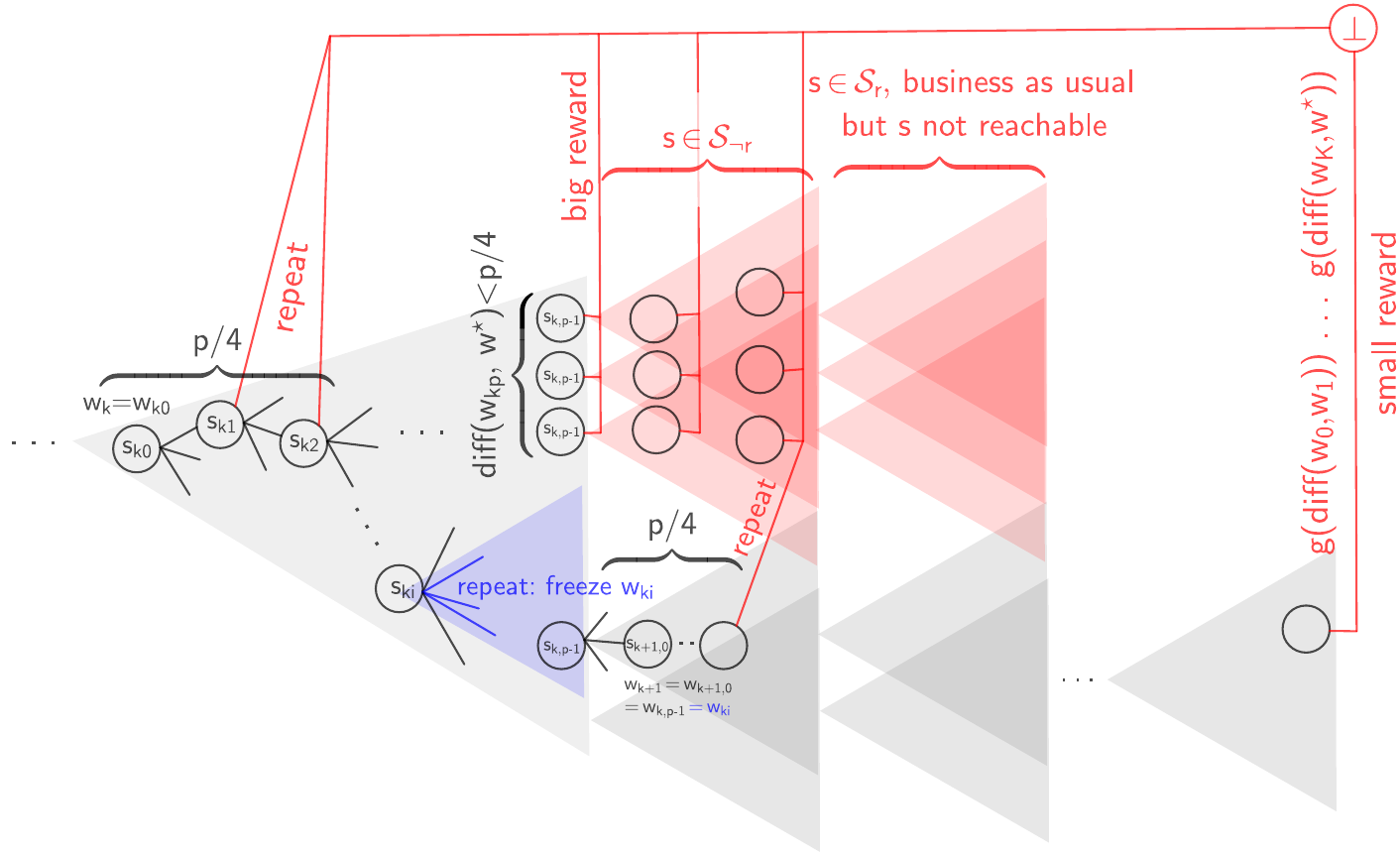}
\centering
\caption{Illustration of an MDP associated with a weight vector $\stheta^\star$. The nodes represent states, which are members of the action tree. Subtrees are illustrated with triangles. Edges represent actions, red edges transit to the episode-over state $\bot$. Unless the action was illegal, there are next-states that an MDP with some other $\stheta^\star$ would have transited to. These states still exist in $M$, but are unreachable, and illustrated with a red triangle. The blue triangle represents a part of the action tree where a legal repeated action freezes the weight corresponding the round. 
Unless written on the edge, there is no reward for the action.
In the figure, $(s_{ki})_{ki}$ represents a path through the state space,
while for $k$, $i$ fixed, $\stheta_{ki}$ represents the weight vector of step $i$ of round $k$.
}\label{fig:mdp-illustration}
\end{figure}

The next step is to show that one can define appropriate feature-maps such that the respective realizability conditions hold, which also means that we will need to compute the optimal value (or action-value) functions and then we will also need to show that a sound MDP planner for the appropriate class of MDPs can be used to derive a sound planner for the abstract game. 

Here, the main idea is that the optimal value corresponding to a state $s$ that is %
in some step $0\le i \le \sd-1$ of round $0\le k\le K-1$ takes the form
\begin{align*}
v^*(s) = f_{\stheta^\star}( (\stheta_i)_{i\in [k]} )
=
g(\diff(\stheta_{0},\stheta_{1})) \dots  
g(\diff(\stheta_{k-1},\stheta_{k}))
g(\diff(\stheta_{k},\stheta^\star))\,,
\end{align*}
where $\stheta_0=\bm{1}$ by convention, $\stheta_i$ for $1\le j < k$ is the weight vector for the corresponding round, while $\stheta_k$ is obtained 
by performing the component manipulations on $\stheta_{k-1}$ prescribed by the action in round $k$ until step $i$, after which, the weight obtained is moved as much as possible towards $\stheta^\star$.
Note that $\stheta_k$ here depends on both $s$ and $\stheta^\star$,
while the other weight vectors only depend on $s$.
In fact, one can write $\stheta_k = A(s) \stheta^\star + b(s)$ for some matrix $A(s)$ and vector $b(s)$ that depend on $s$.
Therefore, 
\begin{align}
v^*(s) = h(s)
g(\diff(\stheta_{k-1}(s),A(s) \stheta^\star + b(s)))
g(\diff(A(s) \stheta^\star + b(s),\stheta^\star))
\label{eq:vquad}
\end{align}
where $h(s) = 
g(\diff(\stheta_{0},\stheta_{1})) \dots  g(\diff(\stheta_{k-2},\stheta_{k-1}))$ is a scalar
that depends only on $s$.
The expression in Eq.~\ref{eq:vquad} is a fourth-order expression of $\stheta^\star$ since $g$ is a quadratic function, 
and while 
$\diff$ is (affine) bilinear in its two arguments and
so it appears that 
$\diff(A(s) \stheta^\star + b(s),\stheta^\star)$ could be quadratic itself, due to the special structure, this expression is still linear in $\stheta^\star$. 
As such, $v^*(s)$ is (roughly\footnote{The precise argument will also include lower-order tensor products.}) a linear function of $(\stheta^\star)^{\otimes 4}$, the fourth-order tensor product of this vector with itself, which gives rise to the definition of $\phiv$ and $\ltheta^\star$, which is (roughly) the flattening of 
$(\stheta^\star)^{\otimes 4}$.
Of course, it remains to verify that $\phiv(s)$ and $\ltheta^\star$ have small norms as required
and also that this definition extends to states that are not reachable from the initial state (to prove the result with global accessibility). 
In fact, it is exactly this second requirement that made us define the deterministic rewards of
$\ip{\phiv(s),\ltheta^\star}$
and the associated transitions to $\bot$. 
(In this case it will be necessary to show that this reward is indeed in the $[0,1]$ interval.)

A similar argument can be used for $q^*$ realizability, and also for $v^*/q^*$ reachable realizability (in which case the reward at unreachable states could be arbitrary).
To finish, one needs to show that a sound MDP planner can be used to implement a sound abstract planner.
For this, note that the steps that an MDP planner makes in the first $8$ rounds of an episode can be directly translated into an admissible weight sequence of length $8$. Further, by construction, the value achieved with this weight sequence is at least as high as the value that the MDP planner would achieve by completing the episode (the function $f_{\stheta^\star}$ and the MDP are such that cutting short a weight sequence obtained from a path in the MDP increases the value of the sequence).

In the remainder of this section, we fill in the gaps of this argument.

\subsection{The MDP construction}
We start with defining $A,\sd$ and $K$ as a function of the horizon $H\ge 81$ and dimension $\ld\ge 31$:
\begin{align}
A&=\sd=\min\left(\max\{ x\in\N_+:\, x^4+x^3+x^2+x+1\le \ld \},\,\floor{H^{1/2}}\right)\,, \label{eq:a-choice}\\
K&=\floor{H/\sd}\,, \nonumber\\
H'&=K\sd \nonumber\,.%
\end{align}
By our definition of $H$-horizon MDPs, any $H'$-horizon MDP for $H'\le H$ is also a $H$-horizon MDP (cf. Assumption~\ref{ass:fixed}). Hence, we shall construct a $H'$-horizon MDP with $H'$ defined above.
For future reference, it will be useful to note that
\begin{align}\label{eq:asymptotic-params}
\begin{split}
A(&=\sd)=\BigTheta\left(H^{1/2} \vmin \ld^{1/4}\right)\,, \qquad
K=\BigTheta(H^{1/2} \vmax H/(\ld^{1/4}))\,, \\
\sd& \ge 2 \, \qquad \text{ and } \qquad
K\ge 9 \,.
\end{split}
\end{align}

Similarly to the abstract game, we fix some $\stheta^\star\in W^\star$ 
(Eq.~\ref{eq:thetastar-distance-to-1}).
In what follows, we define two MDPs $M^v_{\stheta^\star}=(\cS, \cA, Q^v_{\stheta^\star})$ and $M^q_{\stheta^\star}=(\cS, \cA, Q^q_{\stheta^\star})$.

The state and action spaces are the same for all these MDPs.
The superscript $v$ and $q$ indicates which realizability setting the MDP is tailored for. Together with the indices, we drop them and just use $M$ and $Q$ to minimize clutter.
The difference between $M^v_{\stheta^\star}$ and $M^q_{\stheta^\star}$ is minuscule (see Case~\ref{case:rs-def-1}).
As noted beforehand,  $\cA=[\sd]$.

Apart from $\bot$, states in $\cS$ are uniquely identifiable with an action sequence of length at most $K\sd-1$.
Of all action sequences, we need to remove any action sequence that has a ``repeated'' action in the critical first $\minsteps$ steps of any round. 
For $k\ge 0$, let $U_k \subset \cA^k$ be those sequences of in $\cA^k$ which do not have any repeated elements. Then, letting $r= \minsteps$, 
$V=(\bigcup_{i\in[r]} U_i) \cup (\bigcup_{i\in[p-r]} U_r \times \cA^{i})$, we define 
\begin{align*}
\cS & = \{ \bot, () \}  \cup
\bigcup_{0\le k\le K-1} (U_r \times \cA^{p-r})^k \times V\,,
\end{align*}
where $()$ denotes the empty sequence.
The elements of $\cS$ (other than $\bot$) can thus be uniquely identified with a sequence of actions
$(a_{00},\dots,a_{0,p-1},\dots, a_{k0},\dots,a_{ki})$ with $0\le k \le K-1$ and $0\le i \le \sd-1$, 
where the double indexing emphasizes that the steps are grouped into rounds of length $p$, and commas between indices are often dropped to minimize clutter.
For convenience, we let 
$[< k,i]=\{(n,m):\,n\in[0:K-1], m\in[0:\sd-1],\, n\sd+m<k\sd+i\}$ denote the index set in this double indexing,
so that we can write $(a_{nm})_{(n,m)\in [<k,i]}$ for the above action sequence.
Here, we can think of $a_{nm}$ as the action performed in step $m$ of round $n$.

As described beforehand, we associate a ``weight'', an element of $W$, to each state $s\ne \bot$ that corresponds to all the ``flips'' described by the action sequence for $s$.
Let $\stheta:\cS \to W$ be the corresponding map, where we let $\stheta(\bot) = \bm{1}$.
We will also find it useful to introduce $\stheta:\cS \times \cA \to W$, where 
for $(s,a)\in \cS \times \cA$,
$\stheta(s,a)$ is the weight sequence where component $a$ of the last weight vector of $\stheta(s)$ is flipped, except when $s$ is a frozen state or $s=\bot$, in which case $\stheta(s,a) = \stheta(s)$
($s$ is a frozen state when there is a legal repeated action in the actions that correspond to the current round of the state). %

In what follows, we will often find it useful to fix a path, i.e., a complete action sequence of the form $(a_{ki})_{(k,i)\in [0:K-1]\times [0:\sd-1]}\in \cA^{K\sd}$. Note that here we allow all action sequences.
We then describe the behavior of the MDP in terms of its transitions and rewards encountered during this fixed action sequence. Notationally, we refer to the state (deterministically) reached in round $k$, step $i$ for the fixed action sequence as $s_{ki}$.
This means that $s_{00} = ()$, and
for $0\le i \le \sd-1$, $s_{k,i+1} = \gamma_{\stheta^\star}(s_{ki},a_{ki})$
and for $0\le k \le K-1$, $s_{k+1,0} = s_{k\sd}$,
 where $ \gamma_{\stheta^\star}$ is the transition function of the MDP.
Note that the state sequence has extra elements, to help with the notation. In particular, 
$s_{K0} = s_{K-1,\sd} =\bot$.
By a slight abuse of notation, for the fixed action sequence,
we also let $\stheta_{ki} = \stheta( s_{k-1,i},a_{k-1,i} )$ (if $s_{ki}\ne \bot$, $\stheta_{ki}=\stheta(s_{ki})$).
To disambiguate, the notation $\stheta_{ki}$ always uses two indices for $\stheta$, while the notation in the abstract game always uses one.
To match the weight values of the MDP with those of the abstract game, we introduce the shorthand $w_{k}=w_{k0}$.
To complete the definition of $\stheta_{ki}$, we define $\stheta_{00}=\bm{1}$ (similarly to the abstract game's definition of $\stheta_0=\bm{1}$).
We will also find it useful to introduce the function
$\stheta_{\mathrm{last}}:\cS \to W$ which to a given state $s=s_{ki}\ne \bot$ at step $i$ or round $k$ 
assigns the ``last complete weight'' $\stheta_k=\stheta_{k0}$ 
while $\stheta_{\mathrm{last}}(\bot)=\bm{1}$. 
\begin{quotation}
\emph{
The $(k,i)$-indexed notation, such as $s_{ki}$ and $\stheta_{ki}$ (along with other similarly indexed quantities introduced later) is designed to avoid clutter by hiding the implicit dependence on the action sequence, which is assumed to be fixed whenever we use such notations.
The action sequence that is fixed should always be clear from the context. Whenever we state a result concerning these symbols, the result is meant to hold for an arbitrary action sequence.
} %
\end{quotation}

For a state $s\in \cS$, $s\neq\bot$ that is in step $i$ of round $k$, and an action $a\in[A]$, the transition and reward of taking action $a$ in state $s$ leads to the following reward-next state pair $(R', S')$ (which specifies the kernel $Q$ of the MDP):
\begin{subnumcases}{
  \big(R',S'\big)=
  \label{eq:rs-def}
}
  	\left(\ip{\phi,\ltheta^\star},\bot\right) \,, 
		&$\text{if } k>0 \text{ and } \diff(\stheta_k,\stheta^\star) < \sd/4$
			 \label{case:rs-def-1}\\ %
  	(Z, \bot) \,,
		&$\text{else if } i=\sd-1, \diff(\stheta_{k+1},\stheta^\star) < \sd/4$ 
			\label{case:rs-def-2}\\ %
  	(Z, \bot) \,,
		&$\text{else if } k=K-1,\,i=\sd-1 \,(\text{last step})$
			\label{case:rs-def-3}\\ %
  	(0, s_{k,i+1}) \,,
		&$\text{otherwise}$\,,
			\label{case:rs-def-4} %
\end{subnumcases}
Here, the symbols not yet introduced beforehand are defined as follows:
{\em (i)}
$(\stheta_{k'})_{k'\in[k+\one{i=\sd-1}]}$ is the sequence of round-start weights $(\stheta_{k',0})_{k'\in[k+\one{i=\sd-1}]}$ that correspond to state $s$ and action $a$.
If $i=\sd-1$, this sequence also includes the newly ``compiled'' weight $\stheta_{k+1,0}=\stheta(s,a)$.
{\em (ii)} $Z$ has distribution $\Bernoulli(f_{\stheta^\star}((\stheta_{k'})_{k'\in[k+\one{i=\sd-1}]}))$.
{\em (iii)} $\ltheta^\star$ will be defined in Eq.~\ref{eq:ltheta-def}.
{\em (iv)} for feature-maps $\phiv$ and $\phiq$ (defined in Eqs.~\ref{eq:phiv-def}, \ref{eq:phiq-def}), $\phi=\phiv(s_{ki})$ if we are in the $v^\star$-realizable setting (MDP $M^v_{\stheta^\star}$) and $\phi=\phiq(s_{ki},a)$ otherwise.
In either case, the reward in Case~\ref{case:rs-def-1} is in $[0,1]$ by Eq.~\ref{eq:phiv-ip-bound} and Eq.~\ref{eq:phiq-ip-bound}.

Later in the proof, the following lemma will be useful to convert a sound planner for the MDP into a sound planner for the abstract game:
\begin{lemma}\label{lem:can-simulate-mdp-with-abstract-game}
We can simulate an outcome of $(R', S')$ in the MDP using at most one query to the abstract game, if the length, dimensionality, and secret parameters of the game are $K$, $\sd$, and $\stheta^\star$, respectively.
\end{lemma}
\begin{proof}
For $k=0$, $i<\sd-1$, we fall under Case~\ref{case:rs-def-4} and no query to the abstract game is required.
Otherwise, let $l=k+\one{i=\sd-1}>0$, and query the abstract game with 
$(l, (\stheta_{k'})_{k'\in[l]})$.
This is a valid query as $(\stheta_{k'})_{k'\in[l]} \in W^{\circ k}$.
The result to this query allows to determine which case the transition falls under, and it also contains $Z$ (with the required distribution) when the case calls for it.
\end{proof}

As alluded to before, Case~\ref{case:rs-def-1} is somewhat pathological: the transitions are such that if at the end of round $k$ the newly ``compiled'' weight
$\stheta_{k+1,0}$ is close to $\stheta^\star$ 
($\diff(\stheta_{k+1,0},\stheta^\star)<\sd/4$) then the next state is $\bot$.
This means that by following the transitions, it is impossible to arrive at a state $s\in \cS$,
where
Case \ref{case:rs-def-1} would apply.
\begin{lemma}[Case \ref{case:rs-def-1} is unreachable in $M$]\label{lem:unreachable-in-m}
In MDP $M$, for all $s\in\sreach$ and $s'\in\snotreach$,
\[
s'\not\in\reach_M(s)\,
\]
where
\begin{align}
\label{eq:sreach}
\begin{split}
\snotreach = 
\{ s\in \cS \,:\, s\ne \bot \text{ and } \diff(\stheta_{\mathrm{last}}(s),\stheta^\star)<\sd/4\}\,,\qquad
\sreach = \cS\setminus\snotreach\,.
\end{split}
\end{align}
\end{lemma}

We will find some further notation useful to describe essential properties of the MDP states. %
Take any path in the MDP and the corresponding states $(s_{ki})$.
Pick $k$ and $i$ such that $s_{ki}\ne \bot$.
Let
the ``bit mask''
$\fix_{ki}\in \{0,1\}^\sd$ indicate for each component of $\stheta_{ki}$ whether it is fixed (1) in round $k$ at step $i$ or not (0). %
Recall that a component is fixed if either the corresponding action is performed in round $k$ before step $i$, or there was a legal repeated action, in which case all the components are frozen.
Let $\ctflip_{ki}$ be the number of components flipped in round $k$ by step $i$. Because each component can only be flipped at most once in a round, this satisfies
\[
\ctflip_{ki}=\diff(\stheta_{k0}, \stheta_{ki})\,.
\]
Let $\efix_{ki}$ (and $\enotfix_{ki}$) be the number of components that are fixed (and not fixed, respectively) at step $i$ and have the opposite sign of the respective components of $\stheta^\star$.
These are ``error counts''.
(As opposed to $\ctflip_{ki}$ and $\fix_{ki}$, the error counts obviously depend on $\stheta^\star$).
Let the operator $\cdot: \R^d \times \R^d \to \R^d$ 
return the componentwise product of its inputs.
For $i\in \{0,1\}$, let $\neg i = 1-i$, 
which is also extended to binary-valued vectors in a componentwise manner.
The definitions imply the following identities:
\begin{align}\label{eq:eflip-enotflip}
\efix_{ki}&=\frac{1}{2}\left( \ip{\bm{1}, \fix_{ki}}-\ip{\fix_{ki}\cdot\stheta_{ki},\stheta^\star} \right)\,,\\
\enotfix_{ki}&=\frac{1}{2}\left( \ip{\bm{1}, \notfix_{ki}}-\ip{\notfix_{ki}\cdot\stheta_{ki},\stheta^\star} \right)\,.
\label{eq:eflip-enotflip2}
\end{align}
Consider the case when $\fix_{ki}\ne \bm{1}$.
Thanks to $s_{ki}\ne \bot$, the first $i$ actions of round $k$ are unique.
Therefore, in this case, $\ctflip_{ki}=i$.
Furthermore, each unique action adds $1$ to $\ip{\bm{1}, \fix_{ki}}$, thus
$\efix_{ki}\le \ip{\bm{1}, \fix_{ki}}=i=\ctflip_{ki}$.
Similarly, $\enotfix_{ki}\le \ip{\bm{1}, \notfix_{ki}}=\sd-i=\sd-\ctflip_{ki}$. If on the other hand, $\fix_{ki}=\bm{1}$, then $\enotfix_{ki}=0$. This leads to the following result, which will be useful for our calculations: 
\begin{lemma}\label{lem:relationships-betnwee-ctflip-and-efix}
Assuming $s_{ki}\ne\bot$, $\enotfix_{ki}\le \sd-\ctflip_{ki}$, and $\enotfix_{ki}\le \sd-i$.
Furthermore, if $\fix_{ki}\neq \bm{1}$, then the following also hold: $\ctflip_{ki}=i=\ip{\bm{1}, \fix_{ki}}$, and $\efix_{ki}\le \ctflip_{ki}$.
\end{lemma}

\subsection{Defining a policy and calculating its value function}
We now define a deterministic policy $\pi_{\stheta^\star} : \cS\to [A]$, which later will be shown to be the optimal policy. 
The purpose of the current section is merely to compute the value function of this policy.
The policy is defined as follows: 
Let $s_{ki}\in \sreach$ be a state along step $i$ of round $k$ and assume that $s_{ki}\ne \bot$.
Then $\pi_{\stheta^\star}$ greedily flips all the components of $\stheta_{ki}$ that have the wrong sign and are not fixed yet. Once this is done,
$\pi_{\stheta^\star}$ freezes the round by repeating an action. Ties are resolved in a systematic fashion.

More formally, let $\cA_1$ be the set of actions where the component of $\stheta_{ki}$ has not been fixed yet and where $\stheta_{ki}$ disagrees in sign with $\stheta^\star$; 
let $\cA_2$ be the set of actions where the component has been fixed:
\begin{align}
\label{eq:a1-def}
\begin{split}
\cA_1 &= \{a\in[A]\,:\, {(\fix_{ki})}_{a}=0\text{ and }{(\stheta_{ki})}_{a}\neq \stheta^\star_{a}\} \\
\cA_2 &= \{a\in[A]\,:\,{(\fix_{ki})}_{a}=1\}
\end{split}
\end{align}
Then, 
\begin{subnumcases}{
\pi_{\stheta^\star}(s_{ki},\cdot) = \label{eq:pi-def}}
  1\,, & 
  $\text{if } s_{ki}=\bot\,;$\label{case:pi-1}\\ 
  \argmax_{a\in[\sd]} \ip{\phiq(s_{ki},a), \ltheta^\star}\,, & 
  $\text{else if } \diff(\stheta_{k0},\stheta^\star)>\sd/4\,;$\label{case:pi-2}\\  %
  \min \cA_1\,,& 
  $\text{else if }|\cA_1|=\enotfix_{ki}>0\,;$\label{case:pi-3}\\
  \min \cA_2\,,& 
  $\text{else if }|\cA_2|=\ctflip_{ki}>0\,,$\label{case:pi-4}  
\end{subnumcases}
where $\phiq$ is the state-action feature-map defined in Eq.~\ref{eq:phiq-def},  and
$\ltheta^\star$ is defined in Eq~\ref{eq:ltheta-def}.
Note that
$s_{ki}\in\sreach$ and $s_{ki}\ne\bot$ implies that either Case~\ref{case:pi-3} or \ref{case:pi-4} must apply.

With this, the promised result of the section is as follows. 
\begin{lemma}\label{lem:value-of-pi-theta-star}
Assuming $s_{ki}\in\sreach$ and $s_{ki}\ne \bot$, 
 we have
\[
v^{\pi_{\stheta^\star}}(s_{ki})
=
\left(\prod_{k'\in[k]} g(\diff(\stheta_{k'-1,0},\stheta_{k',0})) \right) 
g(\ctflip_{ki} + \enotfix_{ki})g(\efix_{ki})\,.
\]
\end{lemma}
The high level argument underlying this lemma is that the policy reaches the end state
$\bot$, either after reaching the last step of the current, or the next round.
In either cases, the only reward incurred from the current state to the end 
is when the transition to the end state happens.
The definition of this reward can then be invoked to show the result. The detailed proof is as follows:
\begin{proof}
Starting from round $k$ step $i$ and letting $\cA_1$ be as in Eq.~\ref{eq:a1-def},
the policy $\pi_{\stheta^\star}$ flips all the components in $\cA_1$ (that have the wrong sign and 
are not fixed yet).
We note that $\cA_1=\{\}$ if there was a repeated action in this round (which freezes the components).
In this case, 
$\enotfix_{ki}=0$ and
$s_{ki}\ne\bot$ implies the repeated action was legal,
i.e., $i=i+\enotfix_{ki}\ge \minsteps$, and therefore $\stheta_{ki}$ is frozen, thus regardless of $\pi_{\stheta^\star}$, %
$\stheta_{k+1,0}=\stheta_{ki}$, so $\diff(\stheta_{k0}, \stheta_{k+1,0})=\ctflip_{ki}=\ctflip_{ki} + \enotfix_{ki}$.

Otherwise, by definition the first $i+|\cA_1|=i+\enotfix_{ki}=\ctflip_{ki}+\enotfix_{ki}\le\sd$ actions in round $k$ are unique (noting the inequality comes from Lemma~\ref{lem:relationships-betnwee-ctflip-and-efix}).
Furthermore, in this case observe that all
components where $\stheta_{k0}$ differs in sign from $\stheta^\star$ are flipped in round $k$ by step $i+\enotfix_{ki}$: either because it was flipped in the first $i$ steps (and thus setting the relevant component of $\fix_{ki}$ to $1$), or because the action corresponding to the component is in $\cA_1$, and thus flipped by $\pi_{\stheta^\star}$.
Therefore $i+\enotfix_{ki}\ge \diff(\stheta_{k0},\stheta^\star)$
As $\bot\ne s_{ki}\in\sreach$, $\diff(\stheta_{k0},\stheta^\star)\ge \sd/4$.
As $i+\enotfix_{ki}$ is an integer, $i+\enotfix_{ki}\ge \minsteps$.
At step $i+\enotfix_{ki}\ge \minsteps$, 
all the actions in $\cA_1$ are exhausted, and
if there are any remaining steps in the round,
$\pi_{\stheta^\star}$ freezes the round by repeating an action (Case~\ref{case:pi-4}). 
This is a legal action as $i+\enotfix_{ki}\ge \minsteps$.
Therefore $\stheta_{k+1,0}=\stheta_{k,i+\enotfix_{ki}}$.

Regardless of whether $\stheta_{ki}$ is fixed at step $i$,
the number of components that have the wrong sign that are not flipped in round $k$ is exactly $\efix_{ki}$, and therefore
\begin{align*}
\diff(\stheta_{k0}, \stheta_{k+1,0})&=\ctflip_{ki} + \enotfix_{ki}
\\
\diff(\stheta_{k+1,0},\stheta^\star)&=\efix_{ki}
\end{align*}

At the end of round $k$, at step $\sd-1$, either Case~\ref{case:rs-def-2} or \ref{case:rs-def-3} applies and the expectation of the reward is
\begin{align*}
f_{\stheta^\star}\left((\stheta_{k'0})_{k'\in [k+1]}\right)
&=
\left(\prod_{k'\in[k+1]} g(\diff(\stheta_{k'-1,0},\stheta_{k',0})\right) g(\diff(\stheta_{k+1,0},\stheta^\star))\\
&=
\left(\prod_{k'\in[k]} g(\diff(\stheta_{k'-1,0},\stheta_{k',0})\right) g(\ctflip_{ki} + \enotfix_{ki})g(\efix_{ki})\,,
\end{align*}
or Case~\ref{case:rs-def-4} applies and the episode continues with round $k+1$.
In this latter case, $\fix_{k+1,0}=\bm{0}$, $\ctflip_{k+1,0}=0$, $\enotfix_{k+1,0}=\diff(\stheta_{k+1,0},\stheta^\star)=\efix_{ki}$, 
and so in round $k+1$, $\pi_{\stheta^\star}$ sets all the remaining components to match $\stheta^\star$, i.e., $\stheta_{k+2,0}=\stheta^\star$.
The transition at the end of round $k+1$, at step $\sd-1$, then falls either under Case~\ref{case:rs-def-2} or \ref{case:rs-def-3}, and the expectation of the reward is the same as before as $g(0)=1$:
\begin{align*}
f_{\stheta^\star}\left((\stheta_{k'0})_{k'\in [k+2]}\right)
&=
\left(\prod_{k'\in[k+2]} g(\diff(\stheta_{k'-1,0},\stheta_{k',0})\right) g(\diff(\stheta_{k+2,0},\stheta^\star))\\
&=
\left(\prod_{k'\in[k]} g(\diff(\stheta_{k'-1,0},\stheta_{k',0})\right) g(\ctflip_{ki} + \enotfix_{ki})g(\efix_{ki}) g(0)\,,
\end{align*}
As in MDP $M$ any transition with a positive reward expectation transitions to state $\bot$, 
the value of $\pi_{\stheta^\star}$, the expected sum of rewards along the episode, reduces to the expectation of this single reward in the episode.

\end{proof}

\subsection{Showing that $\pi_{\stheta^\star}$ is an optimal policy}

We start with a lemma that will be used to optimize the attainable reward, given the constraints of the MDP.

\begin{lemma}\label{lem:optimise-ks}
For $\sd \ge 2$, $l\ge 2$ integer, let $(x_j)_{j\in[l]}$ be integers $0\le x_j\le \sd$, and let $0\le c_1,\,c_2,\,c_3 \le \sd$ be further integers such that the following all hold: 
\begin{itemize}
\item $c_2\le c_1$ \textbf{or} $c_3=0$;
\item $c_1+c_3\le \sd$;
\item $c_1+c_2+c_3\le \sum_{j\in[l]} x_j$;
\item $c_2 \le \sum_{j\in[2:l]} x_j$;
\item $c_1\le x_1$.
\end{itemize}
\end{lemma}
Then, 
\[
\prod_{j\in[l]} g(x_{j}) \le g(c_1 + c_3)g(c_2)\,.
\]
\begin{proof}
Note that $g(x)>0$ and decreases monotonically for $x\in[0,\sd]$.
First we prove for integers $x\ge y$ such that $1\le x,y \le \sd-1$, it holds that
\begin{equation}\label{eq:g-opt-step}
g(x)g(y)  \le g(x+1)g(y-1)\,.
\end{equation}
Note that $g(x)g(y)-g(x+1)g(y-1)=-\frac{x-y+1}{2\sd^4} \left( \sd(x-2)+y(\sd-x) +x \right)$, and as $x\ge y$, it only remains to prove that $\sd(x-2)+y(\sd-x) +x \ge 0$.
If $x=1$ then $y=1$ and the above holds with equality. Otherwise $x\ge 2$ and all terms are non-negative, finishing the proof of Eq.~\ref{eq:g-opt-step}.

We now claim that for any $0\le y \le x \le p$ integers, $g(x)g(y)\le g( (x+y) \wedge p)$. Since over $[0,p]$, $g$ takes values in $[0,1]$, this clearly holds when either $y=0$ or when $x=p$. Furthermore,
if $1\le y \le x \le p-1$, then 
from Eq.~\ref{eq:g-opt-step} it follows that
$g(x)g(y)\le g(x+1)g(y-1) \le g(x+2) g(y-2) \le g( (x+y)\wedge p ) g( (x+y-p)\vee 0 ) 
\le  g( (x+y)\wedge p )$ where the last inequality follows again because $g(u)\in [0,1]$  when $u\in [0,p]$.

Now,
$g(x_2)g(x_3) g(x_4) \le g( (x_2+x_3)\wedge p ) g(x_4) \le g( (((x_2+x_3)\wedge p) + x_4)\wedge p )
=g( (x_2+x_3+x_4)\wedge p)$. Continuing this way, letting $x_{\ge 2}=\sum_{j\in[2:l] } x_j$, we get
\[
\prod_{j\in[2:l]} g(x_j) 
  \le g(x_{\ge2}\vmin \sd)\,.
\]
Thus,
$
\prod_{j\in[l]}g(x_j) \le g(x_1) g(x_{\ge2}\vmin \sd)
$.

Consider first the case when 
$c_3=0$. Then, by monotonicity of $g$, as $x_1\ge c_1=c_1+c_3$ and $c_2\le x_{\ge2}$,
$\prod_{j\in[l]}g(x_j) \le g(c_1+c_3)g(c_2)$ and we are done.

Now, if $c_3>0$, by assumption 
$c_2\le c_1$. 
In this case, $c_1+c_2+c_3-(x_1\vmin (c_1+c_3))\le x_{\ge2}\vmin \sd$, as (1) $c_1\le (x_1\vmin (c_1+c_3))$ and thus $ c_1+c_2+c_3-(x_1\vmin (c_1+c_3))\le c_2+c_3\le c_1+c_3\le \sd$, while (2) by our assumptions, $x_{\ge 2} \ge c_2$ and $x_1+x_{\ge2}\ge c_1+c_2+c_3$, and therefore $(x_1\vmin (c_1+c_3))+x_{\ge 2} \ge c_1+c_2+c_3$.
By the monotonicity of $g$, we can then conclude that
\[
\prod_{j\in[l]}g(x_j) \le g(x_1\vmin (c_1+c_3)) g(c_1+c_2+c_3-(x_1\vmin (c_1+c_3)))\,.
\]
Let $x'_1$ and $x'_2$ be the above arguments of $g$ in decreasing order, i.e., $x'_1=(x_1\vmin (c_1+c_3)) \vmax (c_1+c_2+c_3-(x_1\vmin (c_1+c_3)))$ and $x'_2=(x_1\vmin (c_1+c_3)) \vmin (c_1+c_2+c_3-(x_1\vmin (c_1+c_3)))$, so that we have
$\prod_{j\in[l]}g(x_j) \le g(x'_1)g(x'_2)$ with $x'_1\le c_1+c_3$ and $x'_1+x'_2=c_1+c_2+c_3$.
Applying Eq.~\ref{eq:g-opt-step} on this product $c_1+c_3-x'_1$ times, we get that
\[
\prod_{j\in[l]}g(x_j) \le g(x'_1)g(x'_2) \le g(c_1+c_3)g(c_2)\,.
\]

\end{proof}

We now show that $\pi_{\stheta^\star}$ is an optimal policy by arguing that its value function matches the optimal value function.
\begin{lemma}[$\pi_{\stheta^\star}$ is an optimal policy]\label{lem:pi-theta-star-optimal}
In MDP $M$,
\[
\forall s\in\cS,\, a\in[A],\quad v^{\pi_{\stheta^\star}}(s)=v^\star(s)\,.
\]
\end{lemma}
\begin{proof}
For $s=\bot$, the claim holds by definition as $v^{\pi_{\stheta^\star}}(\bot)=v^\star(\bot)=0$.
Otherwise, let $s=s_{ki}$ be a state along step $i$ of round $k$.
Let us first consider the case when $s_{ki}\in\snotreach$.
For any action $a$ performed, the transition will happen under Case~\ref{case:rs-def-1}, and the deterministic reward given equals $q^\star(s_{ki},a)$.
If we are in the $v^\star$-realizable setting (for MDP $M^v_{\stheta^\star}$), this reward does not depend on the action and therefore $v^{\pi_{\stheta^\star}}(s)=v^\star(s)$ regardless of $\pi_{\stheta^\star}$.
Otherwise, $\pi_{\stheta^\star}$ chooses an action under Case~\ref{case:pi-2}, which by definition maximizes the reward, so again $v^{\pi_{\stheta^\star}}(s)=v^\star(s)$ in this case as well. 

Let us turn to the case where $s_{ki}\in\sreach$.
There is at most one reward with positive expectation in any round (or none, if an illegal action is taken).
As no state in $\snotreach$ is reachable from $s_{ki}$ (by Lemma~\ref{lem:unreachable-in-m}),
this reward is collected at the end of some round $K'\in[0:K-1]$, at step $\sd-1$, and has expectation
\begin{align*}
f_{\stheta^\star}\left((\stheta_{k'0})_{k'\in [K'+1]}\right)
&=
\left(\prod_{k'\in[K'+1]} g(\diff(\stheta_{k'-1,0},\stheta_{k',0}))\right) g(\diff(\stheta_{K'+1,0},\stheta^\star))\\
&=
\prod_{k'\in[K+1]} g(\diff(\stheta_{k'-1,0},\stheta_{k',0}))\,,
\end{align*}
where, for convenience, we let $\stheta_{k'0}=\stheta^\star$ for $k'\ge K'+2$ (as $g(0)=0$).
This reward expectation is strictly positive (by Lemma~\ref{lem:f-bounds}), so the optimal policy will never take an illegal action.

At round $k$, $g(\diff(\stheta_{k'-1,0},\stheta_{k',0}))$ is fixed for $k'\in[k]$, and the policy can only influence the terms $g(\diff(\stheta_{k'-1,0},\stheta_{k',0}))$ for $k'\in[k+1:K+1]$.
We have by definition that $0\le \diff(\cdot,\cdot) \le \sd$.
In any round, once a component is flipped it cannot be flipped back in the same round.
This implies that
\[
\diff(\stheta_{k0},\stheta_{k+1,0}) = 
\diff(\stheta_{k0},\stheta_{ki}) + 
\diff(\stheta_{ki},\stheta_{k+1,0})
\ge \diff(\stheta_{k0},\stheta_{ki}) =\ctflip_{ki}\,.
\]
On top of this, $\efix_{ki}+\enotfix_{ki}$ components differ in sign between $\stheta_{ki}$ and $\stheta^\star$. %
By the triangle inequality, as %
$\stheta_{K+1,0}=\stheta^\star$, this implies that
\[\sum_{k'\in[k+1:K+1]} 
\diff(\stheta_{k'-1,0},\stheta_{k',0})
\ge \diff(\stheta_{k0},\stheta_{ki}) + 
\diff(\stheta_{ki},\stheta^\star)= \ctflip_{ki}+\efix_{ki}+\enotfix_{ki}\,.\]
Finally, $\efix_{ki}$ of these have already been flipped in round $k$ by step $i$.
These cannot be flipped again in the same round $k$, so they need to be included in some future round, i.e., in $\diff(\stheta_{k'-1,0},\stheta_{k',0}))$ for $k'\ge k+2$: 
\[\sum_{k'\in[k+2:K+1]} \diff(\stheta_{k'-1,0},\stheta_{k',0})
\ge \diff(\stheta_{k+1,1},\stheta^\star)\ge\efix_{ki}\,.\]
By Lemma~\ref{lem:relationships-betnwee-ctflip-and-efix},
\[\enotfix_{ki}\le \sd-\ctflip_{ki}\,,\]
and either $\fix_{ki}=\bm{1}$, implying $\enotfix_{ki}=0$, or $\efix_{ki}\le \ctflip_{ki}$:
\[
\efix_{ki}\le \ctflip_{ki} \quad\quad\text{or}\quad\quad \enotfix_{ki}=0\,.
\]
Therefore, we can apply Lemma~\ref{lem:optimise-ks} with $c_1=\ctflip_{ki},\,c_2=\efix_{ki},\,c_3=\enotfix_{ki}$ to optimize the parameters $(x_{j})_{j\in[K-k+1]}$ where $x_j=\diff(\stheta_{j+k-1},\stheta_{j+k})$, to get that
\[
\prod_{k'\in[k+1:K+1]} g(\diff(\stheta_{k'-1,0},\stheta_{k',0})) \le g(\ctflip_{ki} + \enotfix_{ki})g(\efix_{ki})\,.
\]
Therefore, the optimal policy's expected value (which equals the expectation of the only reward in the episode) is upper bounded as:
\[
v^\star(s_{ki})\ge
\left(\prod_{k'\in[k]} g(\diff(\stheta_{k'-1,0},\stheta_{k',0})\right)
g(\ctflip_{ki} + \enotfix_{ki})g(\efix_{ki})
=v^{\pi_{\stheta^\star}}(s_{ki})
\,,
\]
by Lemma~\ref{lem:value-of-pi-theta-star}. Therefore $v^{\pi_{\stheta^\star}}(s_{ki})=v^\star(s_{ki})$.

\end{proof}

\subsection{Defining $\ltheta^\star$, $\phiv$, and $\phiq$, and showing realizability}\label{sec:phi-def-and-realizability} %

For (column) vectors $M_1, M_2, \ldots$, let us denote by $[M_1,\, M_2,\, \ldots]$ their concatenation $(M_1^\top,\, M_2^\top,\,\ldots)^\top$. %
Let $\flat(M)$ map a tensor of any rank $m$ and any shape $d_1\times d_2\times\ldots\times d_m$ to the vector of dimension $\prod_{i\in[m]} d_i$ by laying out its elements in a canonical order.
Let $\otimes$ denote the tensor product.

We will use the following result to linearize products of vectors:
\begin{lemma}\label{lem:tensorize}
For any positive integer $n$ and any vectors $a_1,a_2,\dots,a_n$ and $b_1,b_2,\dots,b_n$ of equal dimension:
\begin{align*}
\ip{a_1,b_1}\ip{a_2,b_2}\dots \ip{a_n,b_n}=\ip{\flat(a_1\otimes a_2\otimes\dots\otimes a_n),\flat(b_1\otimes b_2\otimes\dots\otimes b_n)}\,.
\end{align*}
\end{lemma}

By Lemma~\ref{lem:unreachable-in-m}, and because $M^v_{\stheta^\star}$ and $M^q_{\stheta^\star}$ have the same transitions and rewards for any state $s\in\sreach$, we do not notationally distinguish 
between $M^v_{\stheta^\star}$ and $M^q_{\stheta^\star}$ when describing the value or action-value functions of these MDPs on states $s\in\sreach$, 
as these are the same in the two MDPs.

We define the feature-map $\phiv: \cS\to \B_{\ld}(1)$
and $\phiq: \cS\times[A]\to \B_{\ld}(1)$.
For state $\bot$, let $\phiv(\bot)=\bm{0}$ and for all actions $a\in[A]$, $\phiq(\bot,a)=\bm{0}$. Realizability immediately holds as $v^\star(\bot)=q^\star(\bot, a)=0=\ip{\bm{0}, \ltheta^\star}$.
For any state $s\in \cS$, $s\ne\bot$, let
$s=s_{ki}$ be a state along step $i$ of round $k$.
Let us introduce the function 
\begin{align}
\label{eq:vstar-ctflip-efix}
v'(s_{ki})
=
\left(\prod_{k'\in[k]} g(\diff(\stheta_{k'-1,0},\stheta_{k',0})\right)
g(\ctflip_{ki} + \enotfix_{ki})g(\efix_{ki})
\,.
\end{align}
By Lemmas~\ref{lem:pi-theta-star-optimal}
and \ref{lem:value-of-pi-theta-star}, it holds that
$v'(s_{ki})=v^\star(s_{ki})$ if $s_{ki}\in\sreach$.
Observe that out of the terms above, only $\enotfix_{ki}$ and $\efix_{ki}$ depends on $\stheta^\star$, and this dependence is linear. In particular, recall 
that $\ctflip_{ki}$ depends only on the actions, and not on $\stheta^\star$.
Combined with the fact that $g$ is a second-order polynomial, $v'(s_{ki})$ is a fourth-order expression in $\stheta^\star$, which can thus be linearized in $1+\sd+\sd^2+\sd^3+\sd^4\le\ld$ dimensions.
Let
$\ntheta^\star=\stheta^\star/\norm{\stheta^\star}_2=\stheta^\star/\sqrt{\sd}$, and %
\begin{align}
\ltheta^\star = 63\left[ 1\,, \ntheta^\star\,, (\ntheta^\star)^{\otimes 2}\,, (\ntheta^\star)^{\otimes 3}\,, (\ntheta^\star)^{\otimes 4},\, \bm{0}^{\ld-(1+\sd+\sd^2+\sd^3+\sd^4)}\right]\,,
\label{eq:ltheta-def}
\end{align}
where $\bm{0}^{\ld-(1+\sd+\sd^2+\sd^3+\sd^4)}$ is a vector of zeros of dimensionality $\ld-(1+\sd+\sd^2+\sd^3+\sd^4)$, serving the purpose to pad the vector to exactly $\ld$ dimensions, as required by the definition.
As $\norm{\ntheta^\star}_2=1$, we have that 
\begin{align*}%
\norm{\ltheta^\star}_2\le 63\cdot 5=315 := \thetabound\,.
\end{align*}

Finally, for $Z_{(0)}\,, Z_{(1)}\,, Z_{(2)}\,, Z_{(3)}\,, Z_{(4)}$ calculated in Appendix~\ref{sec:app:calc-phiv}, if we let
\begin{align}
\phiv(s_{ki})= 
\frac{1}{63}\left(\prod_{k'\in[k]} g(\diff(\stheta_{k'-1,0},\stheta_{k',0}))\right)
\left[ Z_{(0)}\,, Z_{(1)}\,, Z_{(2)}\,, Z_{(3)}\,, Z_{(4)}\,, \bm{0}^{\ld-(1+\sd+\sd^2+\sd^3+\sd^4)}\right]\,,
\label{eq:phiv-def}
\end{align}
then by Eq.~\ref{eq:gxgy},
\begin{align}\label{eq:v-inner-product}
\ip{\phiv(s_{ki}), \ltheta^\star}
&=
\left(\prod_{k'\in[k]} g(\diff(\stheta_{k'-1,0},\stheta_{k',0})\right)
 g(\ctflip_{ki} + \enotfix_{ki})g(\efix_{ki}) = v'(s_{ki})\\
\end{align}

Eq.~\ref{eq:phiv-def} completes the definition of $\phiv$, while Eq.~\ref{eq:v-inner-product} implies that 
\begin{align}
0\le\ip{\phiv(s_{ki}), \ltheta^\star}&\le 1 \label{eq:phiv-ip-bound}\,,
\end{align}
as $v'(s_{ki})$ is a product of $g(\cdot)\in[0,1]$ terms (as $\diff(\cdot,\cdot)\in[0,\sd]$).
Furthermore, combining this with $\norm{\left[ Z_{(0)}\,, Z_{(1)}\,, Z_{(2)}\,, Z_{(3)}\,, Z_{(4)}\right]}_2\le 63$ (by Eq.~\ref{eq:gxgy}),
we have that
\begin{align*}%
\norm{\phiv(s)}_2\le 1 \quad\quad\quad\quad\text{for all } s\in\cS\,,
\end{align*}
which ensures that $\phiv: \cS\to \B_{\ld}(1)$.
We stress that, as required, $\phiv(s_{ki})$ does not depend on $\stheta^\star$.

To show $v^\star$-realizability with these features, i.e., that $v^\star(s_{ki})=\ip{\phiv(s_{ki}),\ltheta^\star}$, we start by pointing out that if $s_{ki}\in\snotreach$ then this immediately holds:

\begin{lemma}\label{lem:realizability-snotreach-immediate}
For any state $s\in\snotreach$ and action $a\in[A]$, regardless of the values of 
$\phiv(s)$, $\phiq(s,a)$, and $\ltheta^\star$, $v^\star$-realizability for $M^v_{\stheta^\star}$ and $q^\star$-realizability for $M^q_{\stheta^\star}$ immediately holds as the transition falls under Case~\ref{case:rs-def-1}:
\begin{align*}
v^\star_{M^v_{\stheta^\star}}(s)&=\ip{\phiv(s),\ltheta^\star}\\
q^\star_{M^q_{\stheta^\star}}(s,a)&=\ip{\phiq(s,a),\ltheta^\star}
\end{align*}
\end{lemma}

Otherwise $s_{ki}\in\sreach$, and $v^\star$-realizability follows from Eq.~\ref{eq:v-inner-product} by recalling that $v'(s_{ki})=v^\star(s_{ki})$ in this case.
We conclude the following lemma from this:
\begin{lemma}\label{lem:v-assumption-holds}
$M^v_{\stheta^\star}$ is $v^\star$-realizable with features $\phiv$:
$(M^v_{\stheta^\star},\phiv)\in\cM^{v^\star}_{\thetabound,d,H,A} \cap \cM^{\mathrm{Pdet}}$.
\end{lemma}

We move on to defining $\phiq$ and showing $q^\star$-realizability for $M^q_{\stheta^\star}$. 
For any state $s\in\cS,\,s\ne \bot$ and action $a\in[A]$, 
let $s=s_{ki}$ be a state along step $i$ of round $k$.
Let $s^a_{k,i+1}$ denote the value taken by $s_{k,i+1}$ if $a_{ki}=a$, and similarly for $\stheta^a_{k,i+1}$.
For $i=\sd-1$ only, let us introduce
\begin{align}
\begin{split}
q'(s_{k,\sd-1}, a) &= \left(\prod_{k'\in[k]} g(\diff(\stheta_{k'-1,0},\stheta_{k',0})\right) g(\diff(\stheta_{k0}, \stheta^a_{k+1,0}))
 g(\diff(\stheta^a_{k+1,0}, \stheta^\star))\,. \label{eq:q-prime}
\end{split}
\end{align}

Let
\[
c(s_{ki},a)=\frac{1}{63}\left(\prod_{k'\in[k]} g(\diff(\stheta_{k'-1,0},\stheta_{k',0})\right) g(\diff(\stheta_{k0}, \stheta^a_{k+1,0}))\,.,
\]
which is a scalar that does not depend on $\stheta^\star$.
The only remaining term in $q'$ has a second-order dependence on $\stheta^\star$. For $X_{(0)},\, X_{(1)},\, X_{(2)}$ calculated in Appendix~\ref{sec:app:calc-phiq}, we let
\begin{subnumcases}{\phiq(s_{ki}, a) = \label{eq:phiq-def}}
  \phiv(s^a_{k,i+1}) & $\text{else if } i<\sd;$\label{case:phiq-2}\\
  c(s_{ki},a) \left[X_{(0)},\, X_{(1)},\, X_{(2)},\, \bm{0}^{\ld-(1+\sd+\sd^2)}\right] & $\text{otherwise},$\label{case:phiq-3}
\end{subnumcases}
where $\bm{0}^{\ld-(1+\sd+\sd^2)}$ is a vector of zeros of dimensionality $\ld-(1+\sd+\sd^2)$.
Then by Eq.~\ref{eq:bigx-for-q}, for $\ltheta^\star$ set according to Eq.~\ref{eq:ltheta-def},
\begin{align}\label{eq:q-inner-product}
\ip{\phiq(s_{k,\sd-1},a), \ltheta^\star}
=
q'(s_{k,\sd-1}, a) \,.
\end{align}
Eq~\ref{eq:phiq-def} completes the definition of $\phiq$, while Eq.~\ref{eq:q-inner-product} together with Eq.~\ref{eq:phiv-ip-bound} implies that 
\begin{align}\label{eq:phiq-ip-bound}
0\le\ip{\phiq(s_{ki},a), \ltheta^\star}\le 1 \quad\quad\quad\text{for all }a\in[A],\,s_{ki}\in\cS,\,s_{ki}\ne \bot\,,
\end{align}
as $q'(s_{k,\sd-1},a)$ is a product of $g(\diff(\cdot,\cdot))\in[0,1]$ terms.
Furthermore, combining this with $\norm{\left[ X_{(0)}\,, X_{(1)}\,, X_{(2)}\right]}_2\le 8$ (by Eq.~\ref{eq:bigx-for-q}), we have that.
\begin{align*}%
\norm{\phiq(s,a)}_2\le 1 \quad\quad\quad\quad\text{for all } s,a\in\cS\times[A]\,,
\end{align*}
which ensures that $\phiq: \cS\times[A]\to \B_{\ld}(1)$, as required.
Again we stress that $\phiv(s_{ki})$ does not depend on $\stheta^\star$.

To show $q^\star$-realizability, we first consider the case when $s_{ki}\in\sreach$ and $i=\sd-1$ i.e., $\phiq(s_{ki},a)$ falls under Case~\ref{case:phiq-3}.
In this case,
\begin{align*}
\ip{\phiq(s_{k,\sd-1},a), \ltheta^\star} &= q'(s_{k,\sd-1}, a)\\
&= \left(\prod_{k'\in[k]} g(\diff(\stheta_{k'-1,0},\stheta_{k',0})\right) g(\diff(\stheta_{k0}, \stheta^a_{k+1,0})) g(\diff(\stheta^a_{k+1,0}, \stheta^\star))\\
&=q^\star(s_{k,\sd-1}, a)\,,
\end{align*}
where the first equality comes from Eq~\ref{eq:q-inner-product}. 
The last equality holds by definition if the transition and reward follows Case~\ref{case:rs-def-2} or \ref{case:rs-def-3}; otherwise under Case~\ref{case:rs-def-4}, it holds since
\begin{align*}
q^\star(s_{k,\sd-1},a) &= v^\star(s^a_{k+1,0}) = q'(s_{k,\sd-1},a)\,,
\end{align*}
where the second equality follows from Lemmas~\ref{lem:value-of-pi-theta-star} and \ref{lem:pi-theta-star-optimal}.

Turning to the case where $s_{ki}\in\sreach$ and $i<\sd-1$, we note that 
$\phiq(s_{ki},a)$ falls under Case~\ref{case:phiq-2}, 
while the transition and reward follows Case~\ref{case:rs-def-4}. Therefore
\begin{align*}
q^\star(s_{ki},a) &= v^\star(s^a_{k,i+1}) 
= 
\ip{\phiv(s^a_{k,i+1}),\ltheta^\star}=
\ip{\phiq(s_{k,i+1},a),\ltheta^\star}\,,
\end{align*}
where the second equality follows from Lemma.~\ref{lem:v-assumption-holds}.

Together with Lemma~\ref{lem:realizability-snotreach-immediate} that proves $q^\star$-realizability for the case of $s_{ki}\in\snotreach$, we conclude that the following holds:
\begin{lemma}\label{lem:q-assumption-holds}
$M^q_{\stheta^\star}$ is $q^\star$-realizable with features $\phiq$:
$(M^q_{\stheta^\star},\phiq)\in\cM^{q^\star}_{\thetabound,d,H,A} \cap \cM^{\mathrm{Pdet}}$.
\end{lemma}

Recall that $\reach(s_{00})\subseteq \sreach$ under either MDP $M^v_{\stheta^\star}$ or $M^q_{\stheta^\star}$ (by Lemma~\ref{lem:unreachable-in-m}),
and that value and action-value functions on such states take the same value for the two MDPs.
Then, combining Lemmas~\ref{lem:v-assumption-holds} and \ref{lem:q-assumption-holds}, we have the following result:
\begin{lemma}\label{lem:vq-assumption-holds}
$M^v_{\stheta^\star}$ is reachable-$v^\star/q^\star$-realizable with features $\phiv$ and $\phiq$:
$(M^v_{\stheta^\star},\phiv,\phiq)\in\cM^{v^\star/q^\star \mathrm{ reach}}_{\thetabound,d,H,A} \cap \cM^{\mathrm{Pdet}}$.
\end{lemma}

\subsection{Reduction to planning in the abstract game}

\begin{proof}[Proof of Theorem~\ref{thm:lb}]
Let $\delta\ge0.01,\,\thetabound\ge315,\,d\ge31,\,H\ge81$.
In what follows, we prove the theorem only for $A$ (and $\sd$) set according to Eq.~\ref{eq:a-choice}.
This is sufficient to prove the result for $A\ge\floor{\sqrt{H}}\vmin 0.8 \ld^{14}\ge \sd$ (Eq.~\ref{eq:a-choice}) as soundness with a lower action count cannot be harder to achieve, since it is always possible to duplicate some actions without changing the difficulty of the problem.

Let $\cP$ be any $\delta$-sound planner with worst-case query cost $\bar N$ for some class class $\cM\cap \cM^{\mathrm{Pdet}}$, where
\[
\cM \in \{ \cM^{v^\star}_{\thetabound,d,H,A}, \cM^{q^\star}_{\thetabound,d,H,A}, \cM^{v^\star/q^\star \mathrm{ reach}}_{\thetabound,d,H,A} \}\,.
\]
We show that $\cP$  gives rise to a sound abstract planner for the abstract game of Section~\ref{sec:abstract-game}
and therefore,  by Theorem~\ref{thm:abstract-lb}, it must use exponentially many queries.

Lemmas~\ref{lem:v-assumption-holds}, \ref{lem:q-assumption-holds}, and \ref{lem:vq-assumption-holds} show that
MDPs $\left(M^v_{\stheta^\star}\right)_{\stheta^\star\in W^\star}$, $\left(M^q_{\stheta^\star}\right)_{\stheta^\star\in W^\star}$, and $\left(M^v_{\stheta^\star}\right)_{\stheta^\star\in W^\star}$ respectively, together with feature-maps $\phiv$ and $\phiq$, belong to these classes.
Therefore, the $\delta$-sound planner $\cP$ satisfies, for any MDP $M$ with parameter $\stheta^\star$ in its class:
\[
v_{M}^{\pi_{M}}(s_{00}) \ge v^\star_{M}(s_{00})-0.01\,,
\]
where $s_{00}$ is the initial state in $M$ and 
$\pi_{M}$ is the policy induced by the interconnection of $\cP$ and MDP $M$. 
Let $\bbP$ and $\E$ be the probability measure and expectation, respectively, induced by this interconnection (as defined in Section~\ref{sec:planning}).
Then,
\begin{align*}
v_{M}^{\pi_{M}}(s_{00})&=\EEg{\sum_{t=1}^H R_t\,\big|\,S_0=s_{00}} \ge v^\star(s_{00})-0.01 \\
\EEg{\sum_{t=1}^{8\sd} R_t + v^\star(S_{8\cdot\sd}) \,\big|\,S_0=s_{00}} &\ge \EEg{\sum_{t=1}^H R_t\,\big|\,S_0=s_{00}} \ge v^\star(s_{00})-0.01 \,,\\
\end{align*}
where we put $\cdot$ in the index of $S$ to signify multiplication: as opposed to $s$, $S$ only has a single index.
It is valid to refer to the state $S_{8\cdot\sd}$ as $K\ge 9$ by Eq.~\ref{eq:asymptotic-params}.
Let us map any partial trajectory $S_0,A_0,S_1,A_1,\ldots,S_{8\cdot\sd-1},A_{8\cdot\sd-1}$ to the sequence
$(\tilde\stheta_i)_{i\in[8]}\in W^{\circ 8}$ as follows.
Let $j\in[8]$ be the smallest index for which $S_{j\cdot\sd-1}=\bot$, or let $j=9$ if no such index exist in $[8]$.
For $i\in[j-1]$, let
$\tilde\stheta_i=\stheta(S_{i\cdot\sd-1}, A_{i\cdot\sd-1})$; for $i\in[j:8]$,
let $\tilde\stheta_i$ be %
any values such that $(\tilde\stheta_i)_{i\in[8]}\in W^{\circ 8}$ (which is always possible as $(\tilde\stheta_i)_{i\in[j-1]}\in W^{\circ j-1}$ when $j>1$).
Let
\[
k^\star=\min\{i\in [8]\,:\, i=8\text{ or }\diff(\tilde\stheta_i, \stheta^\star)<\sd/4\}\,.
\]
Let $R$ be the final reward of an abstract game (with the same parameters $K,\sd,\stheta^\star$) for this sequence $(\tilde\stheta_i)_{i\in[8]}$. By Eq.~\ref{eq:final-abstract-reward},
\[
R=f_{\stheta^\star}(\stheta(S_{i\cdot\sd}))_{i\in[k^\star]}\,.
\]
Observe that if there is an illegal action in the sequence $A_0,\ldots,A_{8\cdot\sd-1}$, then $\sum_{t=1}^{8\sd} R_t + v^\star(S_{8\cdot\sd})=0$.
Otherwise, if
$\diff(\tilde\stheta_i, \stheta^\star)\ge\sd/4$ for all $i\in[8]$,
then 
all transitions leading to $S_{8\cdot \sd}$ fall under Case~\ref{case:rs-def-4} as $K\ge 9$, and
by Lemmas~\ref{lem:value-of-pi-theta-star} and \ref{lem:pi-theta-star-optimal}, $R=v^\star(S_{8\cdot \sd})=\sum_{t=1}^{8\sd} R_t + v^\star(S_{8\cdot\sd})$.
Finally, if $\diff(\tilde\stheta_{k^\star}, \stheta^\star)<\sd/4$, then
$S_{k^\star\cdot\sd}=\bot$, $v^\star(S_{k^\star\cdot\sd})=0$, $R=R_{k^\star\cdot\sd}$, and the rest of the rewards are zero.
Therefore, either way,
\begin{align}\label{eq:why-abstract-sound}
\E[R]\ge\EEg{\sum_{t=1}^{8\sd} R_t + v^\star(S_{8\cdot\sd}) \,\big|\,S_0=s_{00}} \ge v^\star(s_{00})-0.01=f_{\stheta^\star}(()) \,.
\end{align}

Recall that each response to $\cP$'s query to the MDP's simulator, 
as well as the transitions $(R_{t+1},S_{t+1})\sim Q(\cdot|S_t,A_t)$ (for $t\in[0:H-1]$)
can be implemented with at most one simulator call (respectively) to the abstract game (with the same parameters $K,\sd,\stheta^\star$; see Lemma~\ref{lem:can-simulate-mdp-with-abstract-game}).
In expectation, this results in at most $8\sd\bar N + 8\sd$ such queries to the abstract game simulator.
Together with Eq.~\ref{eq:why-abstract-sound}, and noting that the choice of $\stheta^\star\in W^\star$ was arbitrary,
we see that $\cP$ can be used to construct an abstract planner $\cA$ that is sound with worst-case query cost $8\sd\bar N + 8\sd$.
Therefore, by Theorem~\ref{thm:abstract-lb}, and using Eq.~\ref{eq:asymptotic-params},
\begin{align*}
8\sd\bar N + 8\sd &= 2^{\Omega\left(\sd\wedge K\right)}  \\
\bar N &= 2^{\Omega\left(H^{1/2} \vmin \ld^{1/4}\right)}\,.
\end{align*}

\end{proof}

\section{Extending the guarantees of TensorPlan: the Proof of Theorem~\ref{thm:ub}}
\label{sec:ub}

The purpose of this section is to provide a proof of Theorem~\ref{thm:ub}.
As this theorem has three parts depending on the choice of the class of featurized MDPs $\cM$,
we proceed based on this choice.

As mentioned beforehand, \citet{weisz2021query} already proved the theorem
for $\cM=\cM^{v^\star}_{\thetabound,d,H,A}$.
Hence, it remains to show the theorem for
\begin{align*}
\cM =  \cM^{v^\star/q^\star \mathrm{ reach}}_{\thetabound,d,H,A}
\qquad
\text{and }
\qquad
\cM =  \cM^{q^\star}_{\thetabound,d,H,A} \cap \cM^{\mathrm{Pdet}}\,.
\end{align*}
We start with the former case.

First, for easy reference,
we include the pseudocode of \tensorplan, the planner
from \citet{weisz2021query}, which establishes the claim for 
$\cM =  \cM^{v^\star}_{\thetabound,d,H,A}$.
The pseudocode, which can be found in Appendix~\ref{app:tensorplan-pseudocode},
is adjusted in minor ways to fit our conventions.

In the pseudocode, the simulator oracle is represented through the function $\simm$, which returns a reward, next-state, associated-feature triplet as the response to a call (query) of a state and action tuple, as defined in Section~\ref{sec:planning}. %
\tensorplan (TP), as a planner, is defined through the call to function {\tt GetAction}. 
An input to this function indicates whether this function is called for the first state of an episode.
The significance of this is that for the first state, \tensorplan runs a more expensive planning step,
the result of which is reused in subsequent calls to {\tt GetAction} \emph{within} the episode.

From \citet{weisz2021query}, we have the following result:
\begin{theorem}[\citet{weisz2021query}, Theorem 4.2 with Corollary 4.3]
\label{thm:tensorplan-original}
For 
arbitrary positive reals $\delta,\thetabound$ and arbitrary positive integers $d,H$, it holds that
\[
\cC^\star_{\mathrm{LA}}(\cM^{v^\star}_{\thetabound,d,H,A},\delta) =
 O\Big(\mathrm{poly}\Big( \big(\tfrac{dH}{\delta}\big)^A, \thetabound \Big)\Big)\,.
\]
and \tensorplan is a planner that achieves this.
\end{theorem}

Let us now show that the same result also holds for
$\cM =  \cM^{v^\star/q^\star \mathrm{ reach}}_{\thetabound,d,H,A}$.
We actually 
 state and show this result for the case when $q^\star$ realizability over the reachable states is dropped, leading to the class
$\cM^{v^\star \mathrm{ reach}}_{\thetabound,d,H,A}$.
Clearly, it suffices to show the polynomial query complexity bound for $
\cM^{v^\star \mathrm{ reach}}_{\thetabound,d,H,A}$.
\begin{lemma}\label{lem:tensorplan-reach}
For 
arbitrary positive reals $\delta,\thetabound$ and arbitrary positive integers $d,H$, it holds that
\[
\cC^\star_{\mathrm{LA}}(\cM^{v^\star \mathrm{ reach}}_{\thetabound,d,H,A},\delta) =
 O\Big(\mathrm{poly}\Big( \big(\tfrac{dH}{\delta}\big)^A, \thetabound \Big)\Big)\,.
\]
and \tensorplan is a planner that achieves this.
\end{lemma}
\begin{proof}
This result is based on the observation that 
when \tensorplan is used with an MDP $M$ from some initial state $s_0$ of the MDP,
it only collects data from transitions for states in $\reach_M(s_0)$, the set of states that are reachable from $s_0$.
As such, running \tensorplan in $M$ from $s_0$ generates the same joint distribution over queries and transitions as running it in an MDP $M'$ whose state space is restricted to $\reach_M(s_0)$.
Since in $M$ there are no transitions from $\reach_M(s_0)$ to outside of this set, 
MDP $M'$ is well-defined and its optimal value function matches that of $M$ on the states in $\reach_M(s_0)$. 
Therefore $(M',\phiv|_{\reach_M(s_0)})\in \cM^{v^\star}_{\thetabound,d,H,A}$.
Since \tensorplan is $\delta$-optimal with the required polynomial complexity over the latter class,
it induces a $\delta$-optimal policy in $M'$ with polynomial query cost, while, based on the relationship between $M$ and $M'$, this policy is also $\delta$-optimal in $M$.
\end{proof}

To prove 
Theorem~\ref{thm:ub},
it remains to show that the query complexity of 
$ \cM^{q^\star}_{\thetabound,d,H,A} \cap \cM^{\mathrm{Pdet}}$ is also polynomial.
\begin{lemma}\label{lem:tensorplan-qstar}
For 
arbitrary positive reals $\delta,\thetabound$ and arbitrary positive integers $d,H$, it holds that
\[
\cC^\star_{\mathrm{LA}}(  \cM^{q^\star}_{\thetabound,d,H,A} \cap \cM^{\mathrm{Pdet}}, \delta) =
 O\Big(\mathrm{poly}\Big( \big(\tfrac{dH}{\delta}\big)^A, \thetabound \Big)\Big)\,.
\]
and \tensorplan can be adjusted to achieve this.
\end{lemma}
\begin{proof}
\newcommand{\barv}{\bar{v}}
Take any featurized MDP $(M, \phiq)\in\cM^{q^\star}_{\thetabound,d,H,A} \cap \cM^{\mathrm{Pdet}}$.
Let $M=(\cS, [A], Q)$ for some set of states $\cS$ and reward-transition kernel $Q$ and let $q^\star$ be the optimal action-value function in $M$. 
Let $P$ and $R$ be the transition and reward kernels respectively, corresponding to $Q$.
Let $f:\cS \times \cA \to \cS$ be such that $P(f(s,a)|s,a)=1$. This function exists because $M$ has deterministic transitions.
Let $\cS=\{\bot\}\,\cup\,\bigcup_{h=0}^{H-1} \cS_h$ be the decomposition of the state-space of $M$ 
from Assumption~\ref{ass:fixed},
noting that $\cS_h$ are pairwise disjoint. 
Fix an initial state $s_0\in \cS_0$.

The idea of the proof is to construct a new MDP $\bar M$ from $M$ such that
{\em (i)} acting near-optimally in the new MDP implies acting near-optimally in $M$;
{\em (ii)} the optimal value function $\bar v^\star$ of $\bar M$ is linearly realizable with some feature-map $\barphiv$ with a ``small parameter vector'' $\bar \ltheta^\star$;
{\em (iii)} transitions in the new MDP can be simulated by using a simulator of $M$, while this simulator can also provide access to $\barphiv$. Then, one can use \tensorplan with the new MDP to get a good policy in $M$ while keeping the query cost under control.

The new MDP $\bar M=({\bar{\cS}}, [A], \bar{Q})$ is an $H+1$-horizon MDP which
 `delays' rewards and transitions by one step. This will allow us to satisfy all three requirements listed above. In particular, the key to this is that states in the new MDP will be of the form $(s,a)$ where $(s,a)$ is a ``generic''  state-action pair in $M$ and we will ensure that $\bar v^\star( (s,a) ) = q^\star(s,a)$.
The state features in the new MDP can then be essentially chosen to be $\barphiv((s,a))=\phiq(s,a)$.
 
The details are as follows:
Apart from some special cases, the states in this new MDP are pairs of the original MDP's states and actions, of the form $(s,a)$, with the action component corresponding to an action taken in the ``previous step'' in the new MDP. 
Then, when action $a'$ is used in state $(s,a)$, unless $f(s,a)=\bot$ the next state is $(f(s,a),a')$, while the reward incurred comes from $R(\cdot|s,a)$ where $R$ is the reward kernel underlying $Q$.
When $f(s,a)=\bot$, the next state is simply $\bot$, while the reward still comes from $R(\cdot|s,a)$).
At state $\bot$, as before, any action still transitions to $\bot$ with no reward incurred.
Finally, we 
add another state, $(s_0,0)$, to the state of the new MDP so that $\{ (s_0,0) \}$ becomes the set of initial states. The transitions here are as follows:  if action $a$ is taken in state $(s_0,0)$, the next state becomes $(s_0,a)$ with no reward incurred. (The reason for not adding all states from $\cS_1 \times \{0\}$ to $\bar M$ will become clear later.)
In summary, in the new MDP, when an action is taken, the action chosen in the previous time step (and stored as part of the state) is carried out and the new action is stored to be used in the next step.

Let $r(s,a)$ be the expected immediate reward when action $a$ is taken in state $s$.
We claim that by the construction of $\bar M$,
\begin{align}
\label{eq:qsbv}
\bar v^\star( (s,a) )=q^\star(s,a)\,,
\quad
(s,a)\in \cS_h\times [A]\,, 0\le h\le H-1\,.
\end{align}
Indeed, this is trivial for $h=H-1$ as here as $q^\star(s,a)=\barv^\star( (s,a) )=r(s,a)$, as the final step of the episode transitions to $\bot$.
Now, recall that the optimal value functions satisfy the \textbf{Bellman-optimality equations}.
In the case of $M$, these take the form 
\begin{align}
v^\star(s) &= \max_a q^\star(s,a)\,, \label{eq:vmq} \\
q^\star(s,a) & = r(s,a) + v^\star(f(s,a))\,, \qquad (s,a)\in \cS \times [A]\,. \label{eq:qrv}
\end{align}

Denote the same reward for $\bar M$ and state $(s,a)$ and action $a'$ by $\bar r( (s,a), a' )$.
Of course, the Bellman-optimality equations also hold for $\barv^\star$.
With this, if Eq.~\ref{eq:qsbv} holds up to $h+1$, 
for $(s,a)\in \cS_h \times [A]$
we have
\begin{align*}
\barv^\star( (s,a) ) 
& = \max_{a'\in [A]} \bar r( (s,a), a' ) + \barv^\star( (f(s,a),a') ) \tag{Bellman optimality equations for $\bar v^\star$} \\
& = \max_{a'\in [A]} r( s,a ) + q^\star(f(s,a),a') \tag{definition of $\bar r$ and induction hypothesis} \\
& = r( s,a ) + v^\star(f(s,a)) \tag{Eq.~\ref{eq:vmq}} \\
& = q^\star(s,a)\,, \tag{Eq.~\ref{eq:qrv}} 
\end{align*}
finishing the proof of Eq.~\ref{eq:qsbv}.

Now, Eq.~\ref{eq:qsbv} 
suggests just to define $\barphiv( (s,a) ) = \phiq(s,a)$ and $\barphiv(\bot)=\bm{0}$.
This almost works except that we also need to define $\barphiv$ at $(s_0,0)$. To deal with this case, we extend the dimension of the feature space by one, setting 
$\barphiv( (s_0,0) ) = [1, \bm{0}]$.%
\footnote{$\bm{0}$ is a $\ld$-dimensional vector of zeros, and recall that $[1, \bm{0}]=(1, \bm{0}^T)^T$.}
Now, using that $\bar r( (s_0,0),a)=0$ for all $a\in [A]$, 
we have
\begin{align}
\bar v^\star( (s_0,0) ) = \max_a \bar r( (s_0,0),a) + \bar v^\star( (s_0,a) )
= \max_a q^\star(s_0,a) = v^\star(s_0)\,,
\label{eq:vs0}
\end{align}
where the second equality used Eq.~\ref{eq:qsbv} with $(s,a)=(s_0,a)$.
Using this, 
we set $\barltheta^\star = (v^\star(s_0),\ltheta^\star)$ so that $\bar v^\star$ is linearly realizable with the new features.
(This is the point where we exploit that in $\bar M$ there is only a single new initial state: this is why it suffices to add a single extra dimension to the feature space.)

As $\phiq \in\cS\times[A]\to \B_{\ld}(1)$ and $\ltheta^\star\in\B_{\ld}(\thetabound)$ we have that 
$v^\star(s_{0})\le \thetabound$, therefore, by the triangle inequality,
$\bar{\ltheta^\star}\in\B_{\ld+1}(2\thetabound)$.
Using $(M,\phiq) \in  \cM^{q^\star}_{\thetabound,d,H,A}$,
it follows that 
\[
(\bar M,\barphiv)\in  \cM^{v^\star}_{2\thetabound,d+1,H+1,A}\,.
\]

The planning method $\mtpdetq$ for the class $\cM^{q^\star}_{\thetabound,d,H,A} \cap \cM^{\mathrm{Pdet}}$ is designed as follows:
When $\mtpdetq$ is called with $s_0,\phiq(s_0,\cdot)$ at the beginning of an episode,
$\mtpdetq$ calls the {\tt GetAction} method of \tensorplan with 
\[
( (s_0,0),(1,\bm{0}), \true, A,H+1,d+1,\simulatesc',\delta,2B)
\] 
where the pseudocode of 
$\simulatesc'$ is given in Algorithm~\ref{alg:simulate-detq}.

\begin{algorithm}[H]
\caption{
\simulatesc'
}\label{alg:simulate-detq}
\begin{algorithmic}[1]
\State \textbf{Inputs:} $\bar{s}, a'$; \textbf{returns:} rewards, next-states, associated features
\If{$\bar{s}=\bot$}
  \State \Return $(0, \bot,\bm{0})$ %
\ElsIf{$\bar{s}=(s,0)$ for some $s\in\cS$}
  \State \Return $(0, (s,a'), (0,\phiq( s,a' )) )$ \Comment{Note: $\barphiv( (s,a') ) = (0,\phiq( s,a' ))$}
\Else
  \State $(s, a)\gets \bar{s}$
  \State $(R, S')\gets \simulatesc(s, a)$
  \If{$S'=\bot$}
    \State \Return $(R, \bot, \bm{0})$  %
  \Else
    \State \Return $(R, (S', a'), (0, \phiq(S',a')) )$  \Comment{Note: $\barphiv( (S',a') ) = (0,\phiq( S',a' ))$}
  \EndIf
\EndIf
\end{algorithmic}
\end{algorithm}

The action $A_0$ returned by {\tt GetAction} is returned by $\mtpdetq$,
which is then executed in $M$, transitioning to state $S_1$ 
while incurring reward $R_1$
(we set $S_0=s_0$). More generally, for $1\le t \le H$ let $S_t$ be the $t$th state of MDP $M$,
 $R_t$ be the reward associated with transitioning to $S_t$,
and $A_{t-1}$ be the action that led to this transition.
Let $\bar S_0 = (s_0,0)$ be the corresponding state in $\bar M$, and
for $1 \le t \le H-1$ let
 $\bar S_t = (S_t,A_{t-1})$ when $S_t\ne \bot$, and $\bar S_t = \bot$ when $S_t=\bot$.
 We also let $\bar S_{H}=\bot$.

As the interaction between $\bar M$ and $\mtpdetq$ continues, for $1\le t \le H$, 
$\mtpdetq$ is called with $S_t, \phiq(S_t,\cdot)$.
$\mtpdetq$ then calls
{\tt GetAction}  of \tensorplan 
with
\[
( \bar S_t , \barphiv(\bar S_t), \false, A,H+1,d+1,\simulatesc',\delta,2B)\,,
\]
where $\bar S_t$ is constructed as described above, from $S_t$ and from $A_{t-1}$, the action returned by the last call to \tensorplan.{\tt GetAction}, which is stored by $\mtpdetq$ in the global memory.
For $1\le t \le H-1$, the action $A_t$ returned by {\tt GetAction} is used in $M$.

Note that $(R_{t+1},\bar S_{t+1}) \sim  \bar Q(\cdot|\bar S_t,A_t)$ holds for $0\le t \le H-1$:
Formally, if $\bbP$ denotes the distribution induced over interaction sequences 
and $\cF_t$ is the smallest $\sigma$-algebra that makes the history leading up to the choice of $A_t$ (and including $A_t$) measurable,
then 
\begin{align}
\bbP(R_{t+1},\bar S_{t+1}\in \cdot\,|\,\cF_t) = \bar Q(\cdot\,|\,\bar S_t,A_t)\,
\label{eq:barq}
\end{align}
 holds $\bbP$-almost surely for $0\le t\le H-1$.
This relation follows directly from the definitions.
Due to Eq.~\ref{eq:barq},
from the perspective of \tensorplan, the environment is exactly $\bar M$.
Hence, letting $\bar\pi$ denote the policy induced in $\bar M$ by these calls,
and $\bar v^{\bar\pi}$ denote the corresponding value function in $\bar M$,
by Theorem~\ref{thm:tensorplan-original} and Eq.~\ref{eq:vs0},
$\bar v^{\bar \pi}( (s_0,0) ) \ge \bar v^\star( (s_0,0) ) - \delta= v^\star(s_0)-\delta$, while the total expected number of queries 
issued to \simulatesc'
is 
$ O\Big(\mathrm{poly}\Big( \big(\tfrac{(d+1)(H+1)}{\delta}\big)^A, 2\thetabound \Big)\Big)
=
O\Big(\mathrm{poly}\Big( \big(\tfrac{dH}{\delta}\big)^A, \thetabound \Big)\Big)$.
Since in each call to \simulatesc', \simulatesc is called at most once, 
the total number of queries issued  to 
 \simulatesc during the course of an episode satisfies the same bound. It follows that the query cost of the new planner also enjoys this bound. Hence, it remains to show that the new planner is also sound.

To see this, let $\pi$ denote the policy induced by $\mtpdetq$, the planner constructed above for $M$.
Then, because rewards after $H$ steps in $M$ are by definition zero,
\begin{align*}
v^{\pi}(s_0) = \E[ \textstyle \sum_{t=1}^{H} R_t \,|\,S_0=s_{0}] = \bar v^{\bar \pi}((s_0,0))\,,
\end{align*}
where $\E$ is the expectation corresponding to $\bbP$ and the second equality holds because of 
Eq.~\ref{eq:barq} (as noted before, from the point of view of \tensorplan, the environment is $\bar M$ thanks to this identity).
Putting things together thus finishes the proof.
\end{proof}

\begin{proof}[Proof of Theorem~\ref{thm:ub}]
The theorem follows directly from Theorem~\ref{thm:tensorplan-original}
and Lemmas~\ref{lem:tensorplan-reach} and \ref{lem:tensorplan-qstar}.
\end{proof}

\bibliography{linear_fa}

\appendix

\section{Calculating the linear features}\label{app:phiv-z-calc}

\subsection{Calculating feature components of $\phiv$}\label{sec:app:calc-phiv}

We follow the notation of Section~\ref{sec:phi-def-and-realizability}.
In particular, for any state $s\in \cS$, $s\ne\bot$, let
$s=s_{ki}$ be a state along step $i$ of round $k$.
We intend to linearize the expression $g(\ctflip_{ki} + \enotfix_{ki})g(\efix_{ki})$.

Let $x=\ctflip_{ki}+\enotfix_{ki}$ and $y=\efix_{ki}$. Then,
$x$ and $y$ can be written according to Eqs.~\ref{eq:eflip-enotflip},~\ref{eq:eflip-enotflip2} as:
\begin{align*}
\begin{split}
y&=\frac{1}{2}\left( \ip{\bm{1}, \fix_{ki}}-\ip{\fix_{ki}\cdot\ntheta_{ki},\ntheta^\star} \right)
=
\ip{y_{(1,0)}, 1}+\ip{y_{(1,1)}, \ntheta^\star}\\
&\quad\quad\text{for }y_{(1,0)}=\frac{1}{2} \ip{\bm{1}, \fix_{ki}} \text{ and } y_{(1,1)}=-\frac{\sqrt{\sd}}{2}\fix_{ki}\cdot\ntheta_{ki}\\
&\quad\quad\text{with }\norm{y_{(1,0)}}_2,\norm{y_{(1,1)}}_2\le\sd
\end{split}\\
\begin{split}
x&=\ctflip_{ki}+\frac{1}{2}\left( \ip{\bm{1}, \notfix_{ki}}-\ip{\notfix_{ki}\cdot\ntheta_{ki},\ntheta^\star} \right)
=
\ip{x_{(1,0)}, 1}+\ip{x_{(1,1)}, \ntheta^\star}\\
&\quad\quad\text{for }x_{(1,0)}=\ctflip_{ki}+\frac{1}{2}\ip{\bm{1}, \notfix_{ki}} \text{ and } x_{(1,1)}=-\frac{\sqrt{\sd}}{2}\notfix_{ki}\cdot\ntheta_{ki}\\
&\quad\quad\text{with }\norm{x_{(1,0)}}_2,\norm{x_{(1,1)}}_2\le\sd
\end{split}
\end{align*}
Notice that $x_{(\cdot,\cdot)}$ and $y_{(\cdot,\cdot)}$ do not depend on $\stheta^\star$, only on the current state $s_{ki}$.
Furthermore, using Lemma~\ref{lem:tensorize}
\begin{align*}
\begin{split}
x^2&=\ip{x_{(1,0)}^2, 1}+\ip{2x_{(1,0)} x_{(1,1)},\ntheta^\star} + \ip{\flat(x_{(1,1)} \otimes x_{(1,1)}), \flat(\ntheta^\star\otimes \ntheta^\star)}\\
&=\ip{x_{(2,0)},1} + \ip{x_{(2,1)},\ntheta^\star} + \ip{x_{(2,2)},(\ntheta^
\star)^2}\\
&\quad\quad\text{for }x_{(2,0)}=x_{(1,0)}^2 ,\, x_{(2,1)}=2x_{(1,0)} x_{(1,1)},\,
x_{(2,2)}=\flat(x_{(1,1)} \otimes x_{(1,1)}),\\
&\quad\quad\quad\quad\text{and }(\ntheta^\star)^{\otimes 2}=\flat(\ntheta^\star\otimes\ntheta^\star)\\
&\quad\quad\text{with }\norm{x_{(2,0)}}_2,\norm{x_{(2,1)}}_2,\norm{x_{(2,2)}}_2\le2\sd^2
\end{split}\\
\begin{split}
g(x)&=1+x\frac{-2\sd-1}{2\sd^2}+x^2\frac{1}{2\sd^2}
=\ip{X_{(0)},1} + \ip{X_{(1)},\ntheta^\star} + \ip{X_{(2)},(\ntheta^
\star)^2}\\
&\quad\text{for }X_{(0)}=1+\frac{-2\sd-1}{2\sd^2} x_{(1,0)}+\frac{1}{2\sd^2} x_{(2,0)},\,\\
&\quad\quad\quad X_{(1)}=\frac{-2\sd-1}{2\sd^2} x_{(1,1)}+\frac{1}{2\sd^2} x_{(2,1)},\,\text{ and }
X_{(2)}=\frac{1}{2\sd^2} x_{(2,2)}\\
&\quad\quad\text{with }
\norm{X_{(0)}}_2\le 4,\,
\norm{X_{(1)}}_2\le 3,\,
\norm{X_{(2)}}_2\le 1.\\
\end{split}
\end{align*}
and by a similar calculation,
\begin{align*}
\begin{split}
y^2&=
\ip{y_{(2,0)},1} + \ip{y_{(2,1)},\ntheta^\star} + \ip{y_{(2,2)},(\ntheta^
\star)^2}\\
&\quad\quad\text{for }y_{(2,0)}=y_{(1,0)}^2 ,\, y_{(2,1)}=2y_{(1,0)} y_{(1,1)},\,
y_{(2,2)}=\flat(y_{(1,1)} \otimes y_{(1,1)})\\
&\quad\quad\text{with }\norm{y_{(2,0)}}_2,\norm{y_{(2,1)}}_2,\norm{y_{(2,2)}}_2\le2\sd^2
\end{split}\\
\begin{split}
g(y)&=
\ip{Y_{(0)},1} + \ip{Y_{(1)},\ntheta^\star} + \ip{Y_{(2)},(\ntheta^
\star)^2}\\
&\quad\text{for }Y_{(0)}=1+\frac{-2\sd-1}{2\sd^2} y_{(1,0)}+\frac{1}{2\sd^2} y_{(2,0)},\,\\
&\quad\quad\quad Y_{(1)}=\frac{-2\sd-1}{2\sd^2} y_{(1,1)}+\frac{1}{2\sd^2} y_{(2,1)},\,\text{ and }
Y_{(2)}=\frac{1}{2\sd^2} y_{(2,2)} \\
&\quad\quad\text{with }
\norm{Y_{(0)}}_2\le 4,\,
\norm{Y_{(1)}}_2\le 3,\,
\norm{Y_{(2)}}_2\le 1.\\
\end{split}
\end{align*}
Therefore, again using Lemma~\ref{lem:tensorize},
\begin{align}
\begin{split}
\label{eq:gxgy}
g(\ctflip_{ki} + \enotfix_{ki})g(\efix_{ki})&=g(x)g(y)\\
&=\ip{\flat\left(X_{(0)}\otimes Y_{(0)}\right),1} + 
\ip{\flat\left(X_{(0)}\otimes Y_{(1)}+X_{(1)}\otimes Y_{(0)}\right),\ntheta^\star}
\\&\quad+
\ip{\flat\left(X_{(0)}\otimes Y_{(2)}+X_{(1)}\otimes Y_{(1)}+X_{(2)}\otimes Y_{(0)}\right),(\ntheta^\star)^{\otimes 2}}
\\&\quad+
\ip{\flat\left(X_{(1)}\otimes Y_{(2)}+X_{(2)}\otimes Y_{(1)}\right),(\ntheta^\star)^{\otimes 3}}+
\ip{\flat\left(X_{(2)}\otimes Y_{(2)}\right),(\ntheta^\star)^{\otimes 4}}\\
&=\ip{Z_{(0)},1} + \ip{Z_{(1)},\ntheta^\star} + \ip{Z_{(2)},(\ntheta^\star)^{\otimes 2}}
+ \ip{Z_{(3)},(\ntheta^\star)^{\otimes 3}}+ \ip{Z_{(4)},(\ntheta^\star)^{\otimes 4}}
\\
&\quad\text{for }Z_{(0)}=\flat\left(X_{(0)}\otimes Y_{(0)}\right),\\
&\quad\quad\quad Z_{(1)}=\flat\left(X_{(0)}\otimes Y_{(1)}+X_{(1)}\otimes Y_{(0)}\right),\\
&\quad\quad\quad Z_{(2)}=\flat\left(X_{(0)}\otimes Y_{(2)}+X_{(1)}\otimes Y_{(1)}+X_{(2)}\otimes Y_{(0)}\right),\\
&\quad\quad\quad Z_{(3)}=\flat\left(X_{(1)}\otimes Y_{(2)}+X_{(2)}\otimes Y_{(1)}\right),\\
&\quad\quad\quad (\ntheta^\star)^{\otimes 3}=\flat\left(\ntheta^\star\otimes \ntheta^\star\otimes \ntheta^\star\right),\\
&\quad\quad\quad (\ntheta^\star)^{\otimes 4}=\flat\left(\ntheta^\star\otimes \ntheta^\star\otimes \ntheta^\star \otimes \ntheta^\star\right)\\
&\quad\quad\text{with }
\norm{Z_{(0)}}_2\le 16,\,
\norm{Z_{(1)}}_2\le 24,\,
\norm{Z_{(2)}}_2\le 17,\,
\norm{Z_{(3)}}_2\le 6.\\
\end{split}
\end{align}

\subsection{Calculating feature components of $\phiq$}\label{sec:app:calc-phiq}

We follow the notation of Section~\ref{sec:phi-def-and-realizability}.
In particular,
for any state $s\in\cS,\,s\ne \bot$ and action $a\in[A]$, 
let $s=s_{ki}$ be a state along step $i$ of round $k$.
Let $s^a_{k,i+1}$ denote the value taken by $s_{k,i+1}$ if $a_{ki}=a$, and similarly for $\stheta^a_{k,i+1}$.
We intend to linearize the expression $g(\diff(\stheta^a_{k+1,0}, \stheta^\star))$.

Let
$x = \diff\left(\stheta^a_{k+1,0}, \stheta^\star\right)$.
By Eq.~\ref{eq:diff-def},
\begin{align}
\begin{split}
x &= \frac12\left(\sd-\ip{\stheta^a_{k+1,1},\stheta_\star}\right)
= 
\ip{x_{(1,0)}, 1}+\ip{x_{(1,1)}, \stheta^\star}\\
&\quad\quad\text{for }x_{(1,0)}=\frac{1}{2}\sd 
\text{ and } x_{(1,1)}=-\frac{1}{2}\stheta^a_{k+1,1}\\
&\quad\quad\text{with }\norm{x_{(1,0)}}_2,\norm{x_{(1,1)}}_2\le\sd
\label{eq:x-for-q}
\end{split}
\end{align}
By a similar calculation to the previous case,
\begin{align}
\begin{split}
x^2&=\ip{x_{(1,0)}^2, 1}+\ip{2x_{(1,0)} x_{(1,1)},\stheta^\star} + \ip{\flat(x_{(1,1)} \otimes x_{(1,1)}), \flat(\stheta^\star\otimes \stheta^\star)}\\
&=\ip{x_{(2,0)},1} + \ip{x_{(2,1)},\stheta^\star} + \ip{x_{(2,2)},(\stheta^
\star)^2}\\
&\quad\quad\text{for }x_{(2,0)}=x_{(1,0)}^2 ,\, x_{(2,1)}=2x_{(1,0)} x_{(1,1)},\,
x_{(2,2)}=\flat(x_{(1,1)} \otimes x_{(1,1)}),\\
&\quad\quad\quad\quad\text{and }(\stheta^\star)^{\otimes 2}=\flat(\stheta^\star\otimes\stheta^\star)\\
&\quad\quad\text{with }\norm{x_{(2,0)}}_2,\norm{x_{(2,1)}}_2,\norm{x_{(2,2)}}_2\le2\sd^2
\label{eq:x2-for-q}
\end{split}\\
\begin{split}
g(x)&=1+x\frac{-2\sd-1}{2\sd^2}+x^2\frac{1}{2\sd^2}
=\ip{X_{(0)},1} + \ip{X_{(1)},\stheta^\star} + \ip{X_{(2)},(\stheta^
\star)^2}\\
&\quad\text{for }X_{(0)}=1+\frac{-2\sd-1}{2\sd^2} x_{(1,0)}+\frac{1}{2\sd^2} x_{(2,0)},\,\\
&\quad\quad\quad X_{(1)}=\frac{-2\sd-1}{2\sd^2} x_{(1,1)}+\frac{1}{2\sd^2} x_{(2,1)},\,\text{ and }
X_{(2)}=\frac{1}{2\sd^2} x_{(2,2)}\\
&\quad\quad\text{with }
\norm{X_{(0)}}_2\le 4,\,
\norm{X_{(1)}}_2\le 3,\,
\norm{X_{(2)}}_2\le 1.\\
\label{eq:bigx-for-q}
\end{split}
\end{align}

\section{Pseudocode and constants of \tensorplan}\label{app:tensorplan-pseudocode}

\begin{minipage}[t]{0.57\textwidth}
\begin{flushleft}
\begin{algorithm}[H]
\caption{
\tensorplan.{\tt GetAction}
}\label{alg:local}
\begin{algorithmic}[1]
\State \textbf{Inputs:} $s, \phiv, \EpisodeStart$, 
\State $\qquad\qquad$ $A, H, d, \simm,  \delta, \thetabound$
\If{\EpisodeStart} \Comment{Initialize global $\thetafinal$}
	\State \begin{varwidth}[t]{\linewidth}
      \tensorplan.{\tt Init}$($\par
        \hskip\algorithmicindent$s,\phi_v(s), A, H, d, \simm, \delta)$
      \end{varwidth}
\EndIf
\State $\ldd_{\cdot}\gets \mathrm{\approxmeasure}(s,\phiv(s),A,n_2,\simm
)$ \label{line2:avg-calc1} \label{line2:avg-calc} %
\State Access $\thetafinal$ saved by \tensorplan.{\tt Init}
\State \Return $\argmin_{a\in[A]} \left|\ip{\ldd_{a},\left[1,\thetafinal\right]}\right|$ \label{line2:action-choice}
\end{algorithmic}
\end{algorithm}
\end{flushleft}
\end{minipage}
\hfill
\begin{minipage}[t]{0.43\textwidth}
\begin{algorithm}[H]
\caption{\approxmeasure}\label{alg:approx-measure}
\begin{algorithmic}[1]
\State \textbf{Inputs:}
$s,\phiv(s),A,n,\simm$
\For{$a=1$ to $A$}
    \For{$l=1$ to $n$}
      \State $(R_l, S'_l,\phi_v(S'_l))\gets $
      \State $\qquad \qquad \simulatesc(s,a)$
      \State $\tilde \Delta_l\gets \left[R_l, \left(\phiv(S'_l)-\phiv(s)\right)\right]$\label{line:approx-measure}
    \EndFor
    \State $\Delta_{a}:=\frac{1}{n}\sum_{l\in[n]} \tilde \Delta_l$
\EndFor
\State \Return $(\Delta_{a})_{a\in [A]}$ \label{line:avg-calc-approxmeasure}
\end{algorithmic}
\end{algorithm}
\end{minipage}

\begin{algorithm}[t]
\caption{\tensorplan.{\tt Init}}\label{alg:global}
\begin{algorithmic}[1]
\State \textbf{Inputs:}
$s_0,\phiv(s_0),A,H, d, \simm,\delta$
\State $\bfitDelta\gets\{\}$ \Comment{$\bfitDelta$ is a list} \label{line:x-def}
\State Initialize $\zeta, \epsilon,n_1,n_2,n_3$ via equations \eqref{eq:algzetadef}, \eqref{eq:algepsdef}, \eqref{eq:algn1def}, \eqref{eq:algn2def}, \eqref{eq:algn3def}, respectively.
\For{$\tau=1$ to $E_d+2$}
	\State
	Choose any $\theta_\tau \in\argmax_{\theta\in\sol(\bfitDelta)} \ip{\phiv(s_0), \theta}$
		\Comment{Optimistic choice} \label{line:new-theta} \label{line:new-iter}
	\State $\CleanTest \gets\true$
	\For{$t=1$ to $n_1$}			\Comment{$n_1$ rollouts with $\theta_{\tau}$-induced policy}
		\State  $S_{\tau t1}=s_0$ \Comment{Initialize rollout} \label{line:tau-t-j-s0}
		\For{$j=1$ to $H$} \Comment{Stages in episode}
			\State $\ldd_{\tau tj, \cdot} \gets
			\mathrm{\approxmeasure}(
			S_{\tau tj},\phiv(S_{\tau tj}),A,n_2,\simm
			)$ \label{line:avg-calc1}
			\If{
			$\CleanTest$ and
			$\min_{a\in[A]} \left|\ip{\ldd_{\tau tja},\left[1,\theta\right]_\tau}\right|>\tfrac{\delta}{4H}$}\label{line:consistency-test}
			\Comment{Consistency failure?}
				\State $\hd_{\tau tj,\cdot}\gets \mathrm{\approxmeasure}(
				S_{\tau tj},\phiv(S_{\tau tj}),A,n_3,\simm
				)$ \label{line:avg-calc2} \Comment {Refined data}
				\State $\bfitDelta\mathrm{.append}\left(\otimes_{a \in [A]}\hd_{\tau tja}\right)$ \label{line:new-eluder-element} \Comment{Save failure data}
				\State $\CleanTest \gets \false$ \Comment{Not clean anymore}
			\EndIf
            \State $A_{\tau tj}\gets \argmin_{a\in[A]} \left|\ip{\ldd_{\tau tja},\left[1,\theta\right]_\tau}\right|$ \label{line:action-choice} \Comment{Find most consistent action}
			\State $(R_{\tau tj},S_{\tau tj+1}, \phi_v(S_{\tau tj+1})) \gets \simulatesc(S_{\tau tj}, A_{\tau tj})$
			 \Comment{Roll forward} \label{line:simulate-choice}
		\EndFor
	\EndFor
	\State \textbf{if} $\CleanTest$ \textbf{then} \Break \Comment{Success?} \label{line:cleantest-break}
\EndFor
\State Save into global memory $\thetafinal\gets\theta_\tau$\label{line:return}
\end{algorithmic}
\end{algorithm}

\begin{align}
E_d &= \floor{3(d+1)^A\frac{e}{e-1} \ln\left\{ 3+3\left(\frac{2(\thetabound+1)^A 3^A}{H^A\epsilon}\right)^2 \right\}+1}
\label{eq:eddef}\\
\sol\left(\Delta_1, \dots, \Delta_\tau \right) &= \left\{ \theta\in\bR^d \,:\, \norm{\theta}_2\le \thetabound, \forall i\in[\tau]\,: \,\, \left|\ip{\Delta_i,\, \otimes_{a\in[A]}\concat{1\theta}}\right|\le \frac{H^A\epsilon}{2\sqrt{E_d}} \right\} \label{eq:sol-def} \\
\zeta &= \frac{1}{4H}\delta \label{eq:algzetadef} \\
\epsilon &= \left(\frac{\delta}{12H^2}\right)^A/\left(1+\frac{1}{2\sqrt{E_d}}\right) \label{eq:algepsdef}  \\
n_1 &=  \ceil{\frac{32(1+2\thetabound)^2}{\delta^2} \log \frac{E_d+1}{\zeta}} \label{eq:algn1def} \\
n_2 &=  \ceil{\frac{1867H^2(\thetabound+1)^2(d+1)}{2\delta^2}\log (4(E_d+1)n_1HA(d+1)/\zeta)} \label{eq:algn2def} \\
n_3 &= \ceil{\max\left\{n_2, \frac{32(H+1)^2E_d}{\epsilon^2}\log((2(E_d+1)n_1HA))/\zeta\right\}}
\label{eq:algn3def}
\end{align}

\end{document}

%% file: preamble.tex
\usepackage[utf8]{inputenc}
\usepackage[english]{babel}
\usepackage{latexsym}
\usepackage{mathrsfs}
\usepackage{mathtools}
\usepackage{bm}
\usepackage{datetime}
\usepackage{tikz}
\usepackage{listings}
\usepackage{mdframed}
\usepackage{pgfplots}
\usepackage{pgfplotstable}
\usepackage{dsfont}
\usepackage{color}
\usepackage{colortbl}
\usepackage{pifont}
\usepackage[bf,font=small,singlelinecheck=off]{caption}

\usepackage{microtype} %
\usepackage{float}
\usepackage{xfrac} %
\usepackage{xspace}

\renewcommand{\phi}{\varphi}

\usepackage{nicefrac}
\usepackage{wrapfig}

\usepackage[normalem]{ulem}

\newcommand{\W}{\mathcal{W}}

\newcommand{\bR}{\mathbb{R}}

\DeclareMathOperator{\poly}{poly}
\definecolor{emerald}{rgb}{0.31, 0.78, 0.47}

\usepackage{enumerate}
\newtheorem{theorem}{Theorem}[section]

 \newtheorem{lemma}[theorem]{Lemma}

\newtheorem{assumption}[theorem]{Assumption}

\usepackage{aliascnt}

\usepackage{xspace}

\usepackage{enumerate}

\newcommand{\E}{\mathbb E}
\newcommand{\EE}[1]{\mathbb E[#1]}
\newcommand{\EEg}[1]{\mathbb E\left[#1\right]}

\newcommand{\ip}[1]{\left\langle #1 \right\rangle}

\newcommand{\norm}[1]{\left\|#1\right\|}
\newcommand{\R}{\mathbb{R}}
\newcommand{\N}{\mathbb{N}}
\newcommand{\cA}{\mathcal{A}}

\newcommand{\thetabound}{{B}}
\newcommand{\cC}{\mathcal{C}}

\newcommand{\cF}{\mathcal{F}}

\newcommand{\cH}{\mathcal{H}}

\newcommand{\cM}{\mathcal{M}}

\newcommand{\cP}{\mathcal{P}}

\newcommand{\cS}{\mathcal{S}}

\newcommand{\one}[1]{\mathbb{I}\{#1\}}

\newcommand{\simulatesc}{{\textsc{Simulate}}\xspace}
\renewcommand{\epsilon}{\varepsilon}
\newcommand{\ceil}[1]{\left\lceil {#1} \right\rceil}
\newcommand{\floor}[1]{\left\lfloor {#1} \right\rfloor}

\DeclareMathOperator*{\argmin}{arg\ min}
\DeclareMathOperator*{\argmax}{arg\ max}

\newcommand{\bbP}{\mathbb{P}}

\newif\ifsup\suptrue

\usepackage{algorithm}
\usepackage{algpseudocode}